\documentclass{article}

\usepackage[preprint]{neurips_2018}

\pdfoutput=1
\usepackage[utf8]{inputenc} 
\usepackage[T1]{fontenc}    
\usepackage{hyperref}       
\usepackage{url}            
\usepackage{booktabs}       
\usepackage{amsfonts}       
\usepackage{nicefrac}       
\usepackage{microtype}      

\usepackage{comment}
\usepackage{amsfonts}
\usepackage{mathtools}
\usepackage{amsthm}
\usepackage{times}
\usepackage{graphicx} 
\usepackage{subfigure} 

\usepackage{algorithm}

\usepackage{algorithmic}

\newtheorem{theorem}{Theorem}
\newtheorem{proposition}{Proposition}
\newtheorem{lemma}{Lemma}
\newenvironment{sproof}{%
  \proof}{\endproof}
\newcommand{\argmin}{\operatornamewithlimits{argmin}}

\title{Online convex optimization for cumulative constraints}

\author{%
  Jianjun Yuan\\
  Department of Electrical and Computer Engineering\\
  University of Minnesota\\
  Minneapolis, MN, 55455 \\
  \texttt{yuanx270@umn.edu} \\
  \And
  Andrew Lamperski \\
  Department of Electrical and Computer Engineering\\
  University of Minnesota\\
  Minneapolis, MN, 55455 \\
  \texttt{alampers@umn.edu} \\
}

\begin{document}

\maketitle

\begin{abstract}
  We propose the algorithms for online convex
  optimization which lead to cumulative squared constraint violations
  of the form
  $\sum\limits_{t=1}^T\big([g(x_t)]_+\big)^2=O(T^{1-\beta})$, where
  $\beta\in(0,1)$
  .  Previous literature has
  focused on long-term constraints of the form
  $\sum\limits_{t=1}^Tg(x_t)$. There, strictly feasible solutions
  can cancel out the effects of violated constraints.
  In contrast, the new form heavily penalizes large constraint
  violations and cancellation effects cannot occur. 
  Furthermore, useful bounds on the single step constraint violation
  $[g(x_t)]_+$ are derived.
  For convex objectives, our regret bounds generalize
  existing bounds, and for strongly convex objectives we give improved
  regret bounds.
  In numerical experiments, we show that our algorithm closely follows
  the constraint boundary leading to low cumulative violation. 
\end{abstract}

\section{Introduction}

Online optimization is a popular framework for machine learning, with
applications such as  dictionary learning \cite{mairal2009online},
auctions \cite{blum2004online}, classification, and regression
\cite{crammer2006online}.
It has also been influential in the development of 
algorithms in  deep learning such as convolutional neural networks \cite{lecun1998gradient},
deep Q-networks \cite{mnih2015human}, and reinforcement learning \cite{fazel2018global,yuan2017online}.

The general formulation for online convex optimization (OCO)
is as follows: At each time t, we choose a vector $x_t$ in convex set
$S = \{x: g(x)\le 0\}$.
Then we receive a loss function $f_t:S\to R$ drawn from a family
of convex functions and we obtain the loss $f_t(x_t)$. 
In this general setting, there is no constraint on how the sequence of
loss functions $f_t$ is generated.
See \cite{zinkevich2003online} for more details.

The goal is to generate a sequence of $x_t\in S$ for $t = 1,2,..,T$
to minimize the cumulative regret which is defined by:
\begin{equation}
\label{eq::general online}
Regret_T(x^*) = \sum\limits_{t=1}^Tf_t(x_t) - \sum\limits_{t=1}^Tf_t(x^*)
\end{equation}
where $x^*$ is the optimal solution to the following problem:
$\min\limits_{x\in S} \sum\limits_{t=1}^Tf_t(x)$.
According to \cite{cesa2006prediction}, 
the solution to Problem (\ref{eq::general online}) is called Hannan
consistent if $Regret_T(x^*)$ is sublinear 
in $T$. 

For online convex optimization with constraints, a projection operator
is typically applied to the updated variables
in order to make them feasible at each time step \cite{zinkevich2003online,duchi2008efficient,duchi2010composite}.
However, when the constraints are complex, 
the computational burden of the projection may be too high
for online computation.
To circumvent this
dilemma, \cite{mahdavi2012trading} proposed an
algorithm which approximates the true desired projection with a
simpler  closed-form projection. The algorithm gives a cumulative
regret $Regret_T(x^*)$ which is upper bounded by $O(\sqrt{T})$,
but the constraint $g(x_t) \le 0$ may not be satisfied in every time
step.
Instead, the long-term constraint violation satisfies $\sum\limits_{t=1}^Tg(x_t)\le O(T^{3/4})$,
which is useful when we only require the constraint violation to be
non-positive on average: $\lim_{T\to \infty}\sum\limits_{t=1}^Tg(x_t)/T \le 0$.

More recently, \cite{jenatton2016adaptive} proposed an adaptive stepsize version of this algorithm
which can make $Regret_T(x^*)\le O(T^{\max\{\beta,1-\beta\}})$ and 
$\sum\limits_{t=1}^Tg(x_t)\le O(T^{1-\beta/2})$. Here $\beta\in(0,1)$
is a user-determined trade-off parameter.
In related work,
\cite{yu2017online} provides another algorithm
which achieves $O(\sqrt{T})$ regret 
and a bound of $O(\sqrt{T})$ on the long-term constraint violation. 

In this paper, we propose two algorithms for the following two different cases:

\textbf{Convex Case:} The first algorithm is for the convex case,
which also has the user-determined trade-off
as in \cite{jenatton2016adaptive}, while the constraint violation is more strict.
Specifically, we have $Regret_T(x^*)\le O(T^{\max\{\beta,1-\beta\}})$ 
and $\sum\limits_{t=1}^T\big([g(x_t)]_+\big)^2 \le O(T^{1-\beta})$
where $[g(x_t)]_+ =\max\{0,g(x_t)\}$ and  
$\beta\in(0,1)$.
Note the square term heavily penalizes large constraint violations and
constraint violations from one step cannot be canceled out by strictly
feasible steps. 
Additionally, we give a bound on the cumulative
constraint violation 
$\sum\limits_{t=1}^T[g(x_t)]_+ \le O(T^{1-\beta/2})$, which
generalizes the bounds from \cite{mahdavi2012trading,jenatton2016adaptive}. 

In the case of $\beta = 0.5$, which we call "balanced", 
both $Regret_T(x^*)$ and 
$\sum\limits_{t=1}^T([g(x_t)]_+)^2$ have the same upper bound of $O(\sqrt{T})$.
More importantly,
our algorithm guarantees that at each time step, 
the clipped constraint term $[g(x_t)]_+$ is upper bounded by $O(\frac{1}{T^{1/6}})$,
which does not follow from the results of \cite{mahdavi2012trading,jenatton2016adaptive}.
However, our results currently cannot generalize those of
\cite{yu2017online}, 
which has $\sum\limits_{t=1}^Tg(x_t)\le O(\sqrt{T})$. As discussed
below, it is unclear how to extend the work of \cite{yu2017online} to
the clipped constraints, $[g(x_t)]_+$.

\textbf{Strongly Convex Case:} Our second algorithm for strongly convex function $f_t(x)$
gives us the improved upper bounds compared with the previous work in \cite{jenatton2016adaptive}.
Specifically, we have $Regret_T(x^*)\le O(\log(T))$, and
$\sum\limits_{t=1}^T[g(x_t)]_+ \le O(\sqrt{\log(T)T})$. 
The improved bounds match the regret order of standard OCO from
\cite{hazan2007logarithmic}, while maintaining a constraint violation of reasonable order.

We show numerical experiments on three problems. A toy example is used
to compare trajectories of our algorithm with those of
\cite{jenatton2016adaptive,mahdavi2012trading}, and we see that our
algorithm tightly follows the constraints. The algorithms are
also compared on a doubly-stochastic matrix approximation problem
\cite{jenatton2016adaptive} and an economic dispatch problem from
power systems. In these, our algorithms lead to reasonable objective
regret and low cumulative constraint violation.

\section{Problem Formulation }


The basic projected gradient algorithm for
Problem~\eqref{eq::general online} was defined in \cite{zinkevich2003online}. 
At each step,
$t$, the algorithm takes a gradient step with respect to $f_{t}$ and
then projects onto the feasible set.
With some assumptions on $S$ and $f_t$, this algorithm achieves a regret of $O(\sqrt{T})$.

Although the algorithm is simple,  
it needs to solve
a constrained optimization problem at every time step, which might be too time-consuming
for online implementation
when the constraints are complex. 
Specifically, in \cite{zinkevich2003online}, at each iteration $t$, the update rule is:
\begin{equation}
\label{eq::ogd_2003}
\begin{array}{lll}
x_{t+1} &= \Pi_S(x_t-\eta \nabla f_t(x_t)) 
& = \arg \min\limits_{y\in S}\left\|y-(x_t-\eta \nabla f_t(x_t))\right\|^2
\end{array}
\end{equation}
where $\Pi_S$ is the projection operation to the set $S$ and
$\left\|\quad\right\|$ is the $\ell_2$ norm.
%

In order to lower the computational complexity and accelerate the online processing speed,
the work of \cite{mahdavi2012trading} avoids the convex optimization by projecting the variable to a fixed ball $S\subseteq\mathcal{B}$,
which always has a closed-form solution.
That paper gives an online solution for the following problem:
\begin{equation}
\label{eq::original long term version}
\begin{array}{llll}
\underset{x_1,\ldots,x_T\in \mathcal{B}}{\min} & \sum\limits_{t=1}^T f_t(x_t) -
                                    \min\limits_{x\in
                                    S}\sum\limits_{t=1}^T f_t(x)
&s.t. & \sum\limits_{t=1}^T g_i(x_t)\le 0, i = 1,2,...,m 
\end{array}
\end{equation}
where $S = \{x: g_i(x)\le 0, i=1,2,...,m \} \subseteq \mathcal{B}$. It
is assumed that there exist constants $R>0$ and $r<1$ such that
$r\mathbb{K}\subseteq S \subseteq R\mathbb{K}$ with $\mathbb{K}$ being
the unit $\ell_2$ ball centered at the origin and $\mathcal{B} = R\mathbb{K}$.

Compared to Problem (\ref{eq::general online}), which requires that
$x_t \in S$ for all $t$, \eqref{eq::original long term version} implies that 
 only the sum of constraints is
 required. This sum of constraints is known as the \emph{long-term
   constraint}. 
 
To solve this new problem, \cite{mahdavi2012trading} considers the following augmented Lagrangian function at each iteration $t$:
\begin{equation}
\label{eq::pre_l_t}
\mathcal{L}_t(x,\lambda) = f_t(x) + \sum\limits_{i=1}^m \Big\{ \lambda_ig_i(x) - \frac{\sigma \eta}{2}\lambda_i^2 \Big\}
\end{equation}

The update rule is as follows:
\begin{equation}
\label{eq::original update rule}
\begin{array}{ll}
x_{t+1} = \Pi_{\mathcal{B}}(x_t-\eta\nabla_x \mathcal{L}_t(x_t,\lambda_t) ), &
\lambda_{t+1} = \Pi_{[0,+\infty)^m}(\lambda_t + \eta\nabla_{\lambda} \mathcal{L}_t(x_t,\lambda_t) )
\end{array}
\end{equation}
where $\eta$ and $\sigma$ are the pre-determined stepsize and some constant, respectively. 


More recently, an adaptive version was developed in \cite{jenatton2016adaptive},
which has a user-defined trade-off
parameter. The algorithm proposed by \cite{jenatton2016adaptive}
utilizes two different stepsize sequences
to update $x$ and $\lambda$, respectively,  
instead of using a single stepsize $\eta$.

In both algorithms of \cite{mahdavi2012trading} and
\cite{jenatton2016adaptive}, the bound for the violation of the
long-term constraint is that  
$\forall i$, $\sum\limits_{t=1}^T g_i(x_t)\le O(T^{\gamma})$ for some $\gamma \in (0,1)$. 
However, as argued in the last section,
this bound does not enforce that the violation of the constraint $x_t
\in S$ gets small. A situation can arise in which strictly satisfied
constraints at one time step can cancel out violations of the
constraints at other time steps. 
This problem can be rectified by considering clipped constraint, $[g_i(x_t)]_+$, in place
of $g_i(x_t)$.

For convex problems,  
our goal is to bound the term $\sum\limits_{t=1}^T \big([g_i(x_t)]_+\big)^2$,
which, as discussed in the previous section, is more useful for
enforcing small constraint violations,
and also recovers the existing bounds for both $\sum\limits_{t=1}^T [g_i(x_t)]_+$ and $\sum\limits_{t=1}^T g_i(x_t)$.
For strongly convex problems, we also show the improvement on the upper bounds
compared to the results in \cite{jenatton2016adaptive}.

In sum,
in this paper, we want to solve the following problem for the general convex condition:
\begin{equation}
\label{eq::new long term problem}
\begin{array}{llll}
\min\limits_{x_1,x_2,...,x_T\in\mathcal{B}} & \sum\limits_{t=1}^T
                                                f_t(x_t) -
                                                \min\limits_{x\in
                                                S}\sum\limits_{t=1}^T
                                                f_t(x) 
&\quad \quad s.t. & \sum\limits_{t=1}^T \big([g_i(x_t)]_+\big)^2\le O(T^\gamma),
       \forall i 
\end{array}
\end{equation}
where $\gamma \in (0,1)$. The new constraint from
\eqref{eq::new long term problem} is called the \emph{square-clipped
 long-term constraint} (since it is a square-clipped version of the
long-term constraint) or \emph{square-cumulative constraint} (since it
encodes the square-cumulative violation of the constraints).

To solve Problem (\ref{eq::new long term problem}), we change the augmented Lagrangian function $\mathcal{L}_t$ as follows:
\begin{equation}
  \label{eq::new long term lagrangian}
\mathcal{L}_t(x,\lambda) = f_t(x) + \sum\limits_{i=1}^m \Big\{ \lambda_i[g_i(x)]_+ - \frac{\theta_t}{2}\lambda_i^2 \Big\}
\end{equation}



In this paper, we will use the following assumptions as in \cite{mahdavi2012trading}:
1. The convex set $S$ is non-empty, closed, bounded, and can be described by $m$ convex functions as $S = \{x: g_i(x)\le 0, i =1,2,...,m \}$.
2. Both the loss functions $f_t(x)$, $\forall t$ and constraint functions $g_i(x)$, $\forall i$ are Lipschitz continuous in the set $\mathcal{B}$.
That is, $\left\|f_t(x) - f_t(y)\right\| \le L_f\left\|x-y\right\|$, $\left\|g_i(x) - g_i(y)\right\| \le L_g\left\|x-y\right\|$,
$\forall x,y\in \mathcal{B}$ and $\forall t,i$. $G=\max\{L_f,L_g\}$, and 
\begin{equation*}
\begin{array}{ll}
F = \max\limits_{t=1,2,...,T}\max\limits_{x,y\in\mathcal{B}} f_t(x) -f_t(y)\le 2L_fR,  &
D = \max\limits_{i=1,2,...,m}\max\limits_{x\in\mathcal{B}}g_i(x)\le L_gR
\end{array}
\end{equation*}

\section{Algorithm}

\subsection{Convex Case:}

\begin{algorithm}[tb]
    \caption{Generalized Online Convex Optimization with Long-term Constraint}
    \label{alg::alg1}
\begin{algorithmic}[1]
    \STATE {\bfseries Input:} constraints $g_i(x)\le 0,i=1,2,...,m$, stepsize $\eta$, time horizon T, and constant $\sigma>0$.
    \STATE {\bfseries Initialization:} $x_1$ is in the center of the $\mathcal{B}$ .
    \FOR{$t=1$ {\bfseries to} $T$}
    \STATE Input the prediction result $x_t$.
    \STATE Obtain the convex loss function $f_t(x)$ and the loss value $f_t(x_t)$.
    \STATE Calculate a subgradient $\partial_x
    \mathcal{L}_t(x_t,\lambda_t)$, where:
    \begin{equation*}
    \begin{array}{ll}
     \partial_x \mathcal{L}_t(x_t,\lambda_t) = \partial_x f_t(x_t) + \sum\limits_{i=1}^m \lambda_t^i\partial_x ([g_i(x_t)]_+),
     &\partial_x ([g_i(x_t)]_+) =
     \begin{cases}
     0, \quad\mbox{$g_i(x_t) \le0$}\\
     \nabla_x g_i(x_t), \mbox{otherwise}\\
     \end{cases}
    \end{array}
    \end{equation*}
    \STATE Update $x_t$ and $\lambda_t$ as below:
    \begin{equation*}
    \begin{array}{ll}
    x_{t+1} = \Pi_{\mathcal{B}}(x_t-\eta \partial_x \mathcal{L}_t(x_t,\lambda_t)),
    \lambda_{t+1} = \frac{[g(x_{t+1})]_+}{\sigma \eta}
    \end{array}
    \end{equation*}
    \ENDFOR
\end{algorithmic}
\end{algorithm}

The main algorithm for this paper is shown in Algorithm \ref{alg::alg1}.
For
simplicity, we abuse the subgradient notation, denoting a single
element of the subgradient by $\partial_x \mathcal{L}_t(x_t,\lambda_t)$.
Comparing our algorithm with Eq.(\ref{eq::original update rule}), we
can see that the gradient projection step for $x_{t+1}$ is similar, while the update
rule for $\lambda_{t+1}$ is different.
Instead of a projected gradient step, we explicitly maximize
$\mathcal{L}_{t+1}(x_{t+1},\lambda)$ over $\lambda$.
This explicit projection-free update for $\lambda_{t+1}$ is possible
because the constraint clipping guarantees that the maximizer is non-negative. 
Furthermore, this constraint-violation-dependent update helps to
enforce small cumulative and individual constraint
violations. Specific bounds on constraint violation are given in
Theorem \ref{thm::sumOfSquareLongterm} and Lemma \ref{lem:bound_step} below.

Based on the update rule in Algorithm \ref{alg::alg1},
the following theorem gives the upper bounds for both the regret on the loss
and the squared-cumulative constraint violation, $\sum\limits_{t=1}^T\Big([g_i(x_t)]_+\Big)^2$ in Problem \ref{eq::new long term problem}. 
For space purposes, all proofs are contained in the
supplementary material. 

\begin{theorem}
\label{thm::sumOfSquareLongterm}
{\it
Set $\sigma = \frac{(m+1)G^2}{2(1-\alpha)}$, 
$\eta = \frac{1}{G\sqrt{(m+1)RT}}$. 
If we follow the update rule in Algorithm \ref{alg::alg1} with $\alpha\in(0,1)$
and $x^*$ being the optimal solution for $\min\limits_{x\in S}\sum\limits_{t=1}^Tf_t(x)$, 
we have
\begin{equation*}
\begin{array}{ll}
\sum\limits_{t=1}^T\Big( f_t(x_t) - f_t(x^*)\Big)\le O(\sqrt{T}),& 
\sum\limits_{t=1}^T\Big([g_i(x_t)]_+\Big)^2 \le O(\sqrt{T}), \forall i \in \{1,2,...,m\}
\end{array}
\end{equation*}
}
\end{theorem}

From Theorem \ref{thm::sumOfSquareLongterm}, we can see that by
setting appropriate stepsize, $\eta$, and constant, $\sigma$,
we can obtain the upper bound for the regret of the loss function being less than or equal to $O(\sqrt{T})$,
which is also shown in \cite{mahdavi2012trading} \cite{jenatton2016adaptive}.
The main difference of the Theorem \ref{thm::sumOfSquareLongterm} is that 
previous results of \cite{mahdavi2012trading} \cite{jenatton2016adaptive}
all obtain the upper bound for the 
long-term constraint $\sum\limits_{t=1}^T g_i(x_t)$, 
while here the upper bound for 
the constraint violation of the form $\sum\limits_{t=1}^T\Big([g_i(x_t)]_+\Big)^2$ is achieved. 
Also note that the stepsize depends on $T$, which may not be available. 
In this case, we can use the 'doubling trick' described in the book \cite{cesa2006prediction}
to transfer our $T$-dependent algorithm into $T$-free one with a worsening factor of $\sqrt{2}/(\sqrt{2}-1)$. 

The proposed algorithm and the resulting bound are useful for two reasons:
1. The square-cumulative constraint implies a bound on the
cumulative constraint violation,
$\sum\limits_{t=1}^T[g_i(x_t)]_+$, while enforcing larger penalties
for large violations. 
%
2. The proposed algorithm can also upper bound the constraint violation for each single step $[g_i(x_t)]_+$, 
which is not bounded in the previous literature.

The next results show how to bound constraint violations at each
step. 

\begin{lemma}
  \label{lem:bound_step}
  {\it
     If there is only one differentiable constraint function $g(x)$ with Lipschitz continuous gradient parameter $L$, and we run the Algorithm \ref{alg::alg1}
     with the parameters in Theorem \ref{thm::sumOfSquareLongterm} and large enough $T$, we have
     \begin{equation*}
     \begin{array}{lll}
     [g(x_t)]_+ \le O(\frac{1}{T^{1/6}}),& \forall t \in \{1,2,...,T\},
     &if \quad [g(x_1)]_+ \le O(\frac{1}{T^{1/6}}).
     \end{array}
     \end{equation*}

  }
\end{lemma}

Lemma \ref{lem:bound_step} only considers single constraint case. 
For case of multiple differentiable constraints, 
we have the following:
\begin{proposition}
\label{prop::bound_step_max}
  {\it 
     For multiple differentiable constraint functions $g_i(x)$, $i\in\{1,2,...,m\}$ with Lipschitz continuous gradient parameters $L_i$,
     if we use $\bar{g}(x) = \log\Big(\sum\limits_{i=1}^m \exp{g_i(x)}\Big)$
     as the constraint function in Algorithm \ref{alg::alg1}, then for
     large enough $T$, we have 
     \begin{equation*}
     \begin{array}{lll}
     [g_i(x_t)]_+ \le O(\frac{1}{T^{1/6}}),& \forall i, t,
     &if \quad [\bar{g}(x_1)]_+ \le O(\frac{1}{T^{1/6}}).
     \end{array}
     \end{equation*}

  }
\end{proposition}

Clearly, both Lemma \ref{lem:bound_step} and Proposition \ref{prop::bound_step_max} 
only deal with differentiable functions.
For a non-differentiable function $g(x)$,
we can first use a differentiable function $\bar{g}(x)$ to approximate the $g(x)$
with $\bar{g}(x)\ge g(x)$, and then apply the previous Lemma \ref{lem:bound_step}
and Proposition \ref{prop::bound_step_max} to upper bound
each individual $g_i(x_t)$. 
Many non-smooth convex functions can be approximated in this way as
shown in \cite{nesterov2005smooth}. 

\subsection{Strongly Convex Case:}

For $f_t(x)$ to be strongly convex, the Algorithm \ref{alg::alg1} is still valid. 
But in order to have lower upper bounds for both objective regret and the clipped long-term constraint
 $\sum\limits_{t=1}^T [g_i(x_t)]_+$ compared with Proposition \ref{prop::tradeOffLossAndConstraint} in next section, 
we need to use time-varying stepsize as 
the one used in \cite{hazan2007logarithmic}. Thus, we modify the update rule 
of $x_t$, $\lambda_t$ to have time-varying stepsize 
as below:
\begin{equation}
\label{eq::update_strongly_convex}
\begin{array}{ll}

x_{t+1} = \Pi_{\mathcal{B}}(x_t-\eta_t \partial_x \mathcal{L}_t(x_t,\lambda_t)),&
\lambda_{t+1} = \frac{[g(x_{t+1})]_+}{\theta_{t+1}}.
\end{array}
\end{equation}

If we replace the update rule in Algorithm \ref{alg::alg1} with Eq.(\ref{eq::update_strongly_convex}),
we can obtain the following theorem: 

\begin{theorem}
\label{thm::stronglyconvex}
{\it
Assume $f_t(x)$ has strongly convexity parameter $H_1$.
If we set $\eta_t = \frac{1}{H_1(t+1)}$, $\theta_t=\eta_t(m+1)G^2$,
follow the new update rule in Eq.(\ref{eq::update_strongly_convex}),
and $x^*$ being the optimal solution for $\min\limits_{x\in S}\sum\limits_{t=1}^Tf_t(x)$, 
for $\forall i \in \{1,2,...,m\}$,
we have
\begin{equation*}
\begin{array}{ll}
\sum\limits_{t=1}^T\Big( f_t(x_t) - f_t(x^*)\Big)\le O(\log(T)),
&\sum\limits_{t=1}^Tg_i(x_t)\le\sum\limits_{t=1}^T[g_i(x_t)]_+ \le O(\sqrt{\log(T) T}).
\end{array}
\end{equation*}
}
\end{theorem}

The paper \cite{jenatton2016adaptive} also has a discussion of
strongly convex functions, 
but only provides a bound similar to the convex one.
Theorem \ref{thm::stronglyconvex} shows the improved bounds for both objective regret and
the constraint violation.
On one hand the objective regret is consistent with the standard OCO result in \cite{hazan2007logarithmic}, 
and on the other the constraint violation is further reduced compared with the result in \cite{jenatton2016adaptive}.

\section{Relation with Previous Results}

In this section, we extend Theorem
\ref{thm::sumOfSquareLongterm} to enable direct
comparison with the results from  \cite{mahdavi2012trading}
\cite{jenatton2016adaptive}. In particular, it is shown how
Algorithm~\ref{alg::alg1} recovers the existing regret bounds, while
the use of the new augmented Lagrangian \eqref{eq::new long term
  lagrangian} in the previous algorithms also provides regret bounds
for the clipped constraint case. 

The first result puts a bound on the clipped long-term constraint, rather
than the sum-of-squares that appears in
Theorem~\ref{thm::sumOfSquareLongterm}. This will allow more direct
comparisons with the existing results. 

\begin{proposition}
\label{prop::similarResultTo2012}
  {\it 
      If $\sigma = \frac{(m+1)G^2}{2(1-\alpha)}$, 
      $\eta = O(\frac{1}{\sqrt{T}})$,
      $\alpha\in (0,1)$, and
      $x^* = \underset{x\in S}\argmin\sum\limits_{t=1}^Tf_t(x)$, 
      then the result of Algorithm \ref{alg::alg1} satisfies
      \begin{equation*}
      \begin{array}{ll}
      \sum\limits_{t=1}^T\Big( f_t(x_t) - f_t(x^*)\Big)\le O(\sqrt{T}),
      &\sum\limits_{t=1}^Tg_i(x_t)\le \sum\limits_{t=1}^T[g_i(x_t)]_+ \le O(T^{3/4}), \forall i \in \{1,2,...,m\}
      \end{array}
      \end{equation*}
  }
\end{proposition}

This result shows that our algorithm generalizes the
regret and long-term constraint bounds of \cite{mahdavi2012trading}.
%

The next result shows that by changing our constant stepsize accordingly,
with the Algorithm \ref{alg::alg1}, we can 
achieve the user-defined trade-off from \cite{jenatton2016adaptive}.
Furthermore, we also include the squared version and clipped
constraint violations.

\begin{proposition}
\label{prop::tradeOffLossAndConstraint}
  {\it 
      If $\sigma = \frac{(m+1)G^2}{2(1-\alpha)}$, $\eta =
      O(\frac{1}{T^{\beta}})$, $\alpha\in(0,1)$, $\beta \in(0,1)$,
      and $x^* = \underset{x\in S}\argmin\sum\limits_{t=1}^Tf_t(x)$, 
      then the result of  Algorithm~\ref{alg::alg1} satisfies
      \begin{equation*}
      \begin{array}{lll}
      \sum\limits_{t=1}^T\Big( f_t(x_t) - f_t(x^*)\Big)\le O(T^{max\{\beta,1-\beta\}}),\\
      \sum\limits_{t=1}^Tg_i(x_t)\le \sum\limits_{t=1}^T[g_i(x_t)]_+ \le O(T^{1-\beta/2}),
      &\sum\limits_{t=1}^T([g_i(x_t)]_+)^2\le O(T^{1-\beta}), \forall i \in \{1,2,...,m\}
      \end{array}
      \end{equation*}
  }
\end{proposition}

Proposition \ref{prop::tradeOffLossAndConstraint} 
 provides a systematic way 
to balance the regret of the objective and the constraint violation.
Next, we will show that previous algorithms can use our proposed augmented Lagrangian function to 
have their own clipped long-term constraint bound.
\begin{proposition}
\label{prop::true_violation_bound_2011}
{\it
       If we run Algorithm 1 in \cite{mahdavi2012trading} with the augmented Lagrangian formula 
       defined in Eq.(\ref{eq::new long term lagrangian}), the result satisfies
      \begin{equation*}
      \begin{array}{ll}
      \sum\limits_{t=1}^T\Big( f_t(x_t) - f_t(x^*)\Big)\le O(\sqrt{T}),
      &\sum\limits_{t=1}^Tg_i(x_t)\le \sum\limits_{t=1}^T[g_i(x_t)]_+ \le O(T^{3/4}), \forall i \in \{1,2,...,m\}.
      \end{array}
      \end{equation*}
  }
\end{proposition}

For the update rule proposed in \cite{jenatton2016adaptive},
we need to change the 
$\mathcal{L}_t(x,\lambda)$ to the following one:
\begin{equation}
\label{eq::new_lag_2016}
  \mathcal{L}_t(x,\lambda) = f_t(x)+\lambda [g(x)]_+ - \frac{\theta_t}{2}\lambda^2
\end{equation}
where $g(x) = \max_{i\in \{1,\dots,m\}} g_i(x)$.

\begin{proposition}
\label{prop::true_violation_bound_2016}
  {\it 
      If we use the update rule and the parameter choices in \cite{jenatton2016adaptive} with
      the augmented Lagrangian in Eq.(\ref{eq::new_lag_2016}),
      then $\forall i \in \{1,...,m\}$, we have
      \begin{equation*}
      \begin{array}{ll}
      \sum\limits_{t=1}^T\Big( f_t(x_t) - f_t(x^*)\Big)\le O(T^{max\{\beta,1-\beta\}}),
      &\sum\limits_{t=1}^Tg_i(x_t)\le \sum\limits_{t=1}^T[g_i(x_t)]_+ \le O(T^{1-\beta/2}).
      \end{array}
      \end{equation*}
  }
\end{proposition}



\begin{figure}
\vskip 0.0in
  \centering
  \subfigure{
    \label{fig::toy_traj_beta_0.5} 
    \includegraphics[height=3.6cm]{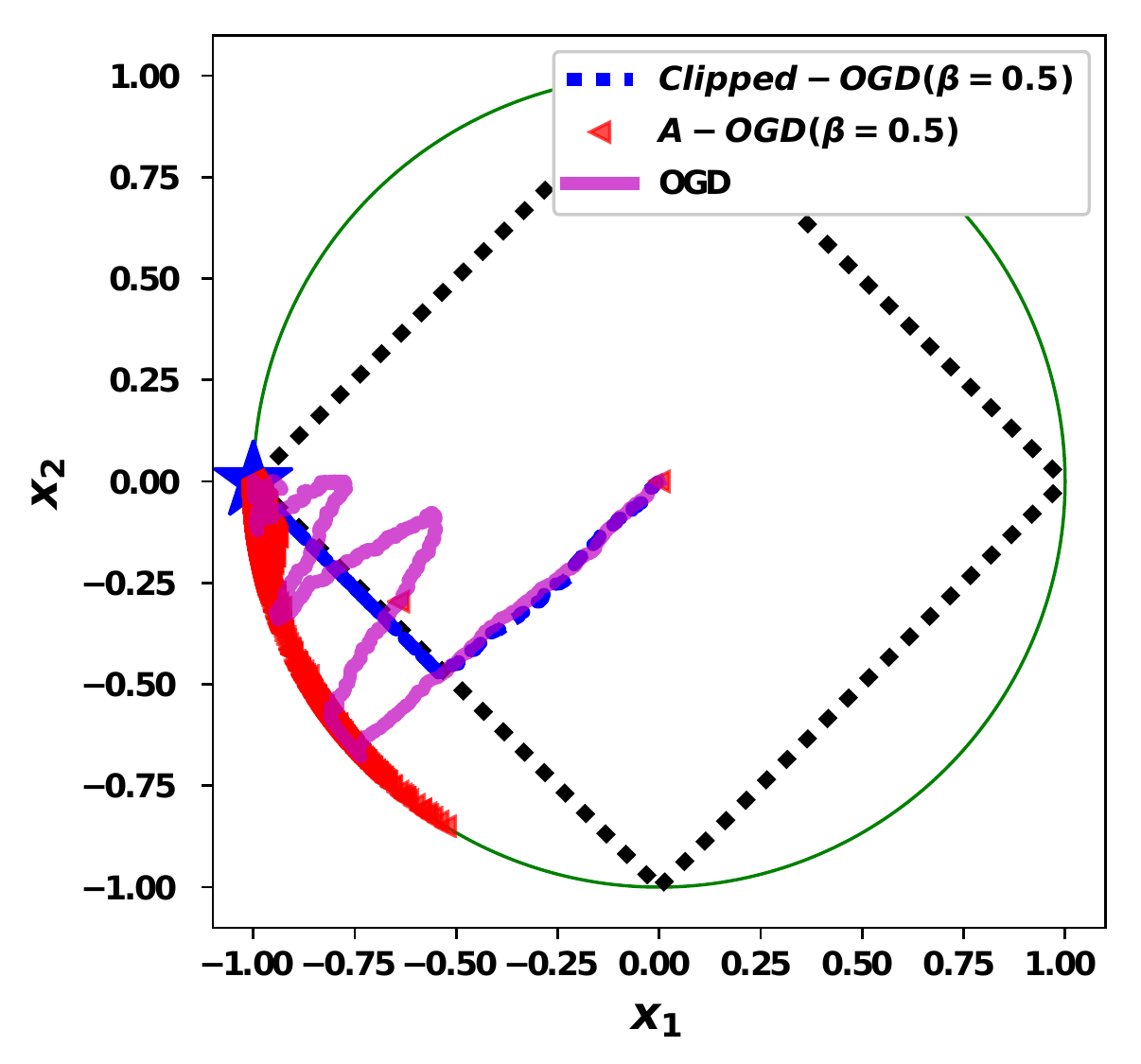}}
  \hspace{.6in}
  \subfigure{
    \label{fig::toy_traj_beta_2/3} 
    \includegraphics[height=3.6cm]{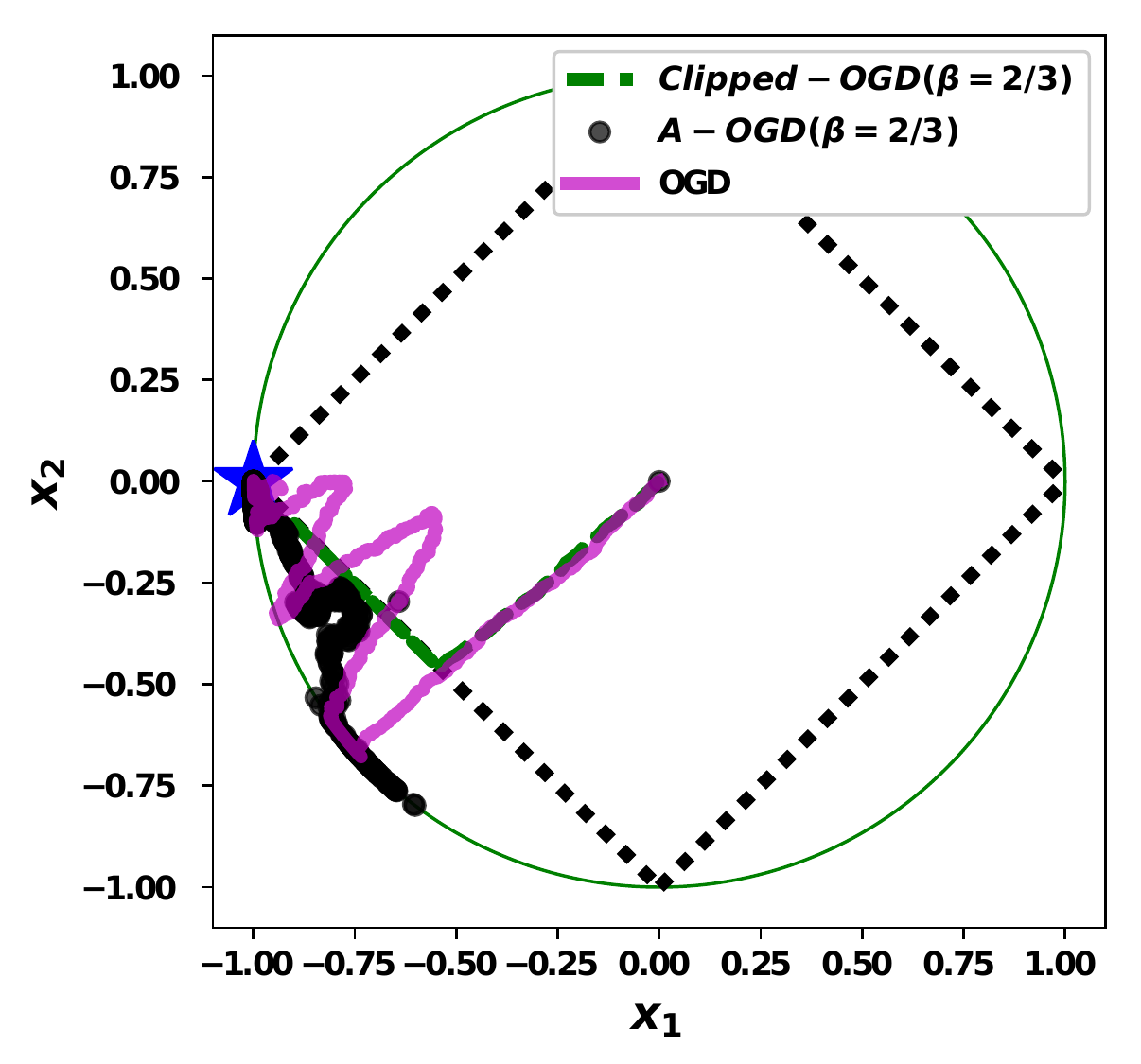}}
  \caption{Toy Example Results: Trajectories generated by different algorithms. 
    Note how trajectories generated by Clipped-OGD follow the
    desired constraints tightly. In contrast, OGD oscillates around
    the true constraints, and A-OGD closely follows the boundary of
    the outer ball.}
  \label{fig::toy_traj} 
\vskip 0.in
\end{figure}

Propositions \ref{prop::true_violation_bound_2011} and
\ref{prop::true_violation_bound_2016} show that clipped long-term
constraints can be bounded by combining the algorithms of
\cite{mahdavi2012trading,jenatton2016adaptive} with our augmented
Lagrangian. 
Although these results are similar in part to our Propositions \ref{prop::similarResultTo2012} and \ref{prop::tradeOffLossAndConstraint},
they do not imply the results in Theorems \ref{thm::sumOfSquareLongterm} and \ref{thm::stronglyconvex} 
as well as the new single step constraint violation bound in Lemma \ref{lem:bound_step}, which are our key contributions.
Based on Propositions \ref{prop::true_violation_bound_2011} and \ref{prop::true_violation_bound_2016},
it is natural to ask whether we could apply our new augmented Lagrangian formula (\ref{eq::new long term lagrangian})
to the recent work in \cite{yu2017online} .
%
Unfortunately, we have not found a way to do so. 


Furthermore, since $\Big([g_i(x_t)]_+\Big)^2$ is also convex, 
we could define $\tilde{g}_i(x_t) = \Big([g_i(x_t)]_+\Big)^2$ 
and apply the previous algorithms \cite{mahdavi2012trading} \cite{jenatton2016adaptive} and \cite{yu2017online}.
This will result in the upper bounds of $O(T^{3/4})$ \cite{mahdavi2012trading} and $O(T^{1-\beta/2})$ \cite{jenatton2016adaptive},
which are worse than our upper bounds of $O(T^{1/2})$ (Theorem \ref{thm::sumOfSquareLongterm})
and $O(T^{1-\beta})$ ( Proposition \ref{prop::tradeOffLossAndConstraint}).
Note that the algorithm in \cite{yu2017online} cannot be applied 
since the clipped constraints do not satisfy the required Slater condition.

\section{Experiments}

\begin{figure}
\vskip 0.0in
  \centering
  \subfigure[]{
    \label{fig::doubly_clip_con} 
    \includegraphics[height=3.cm]{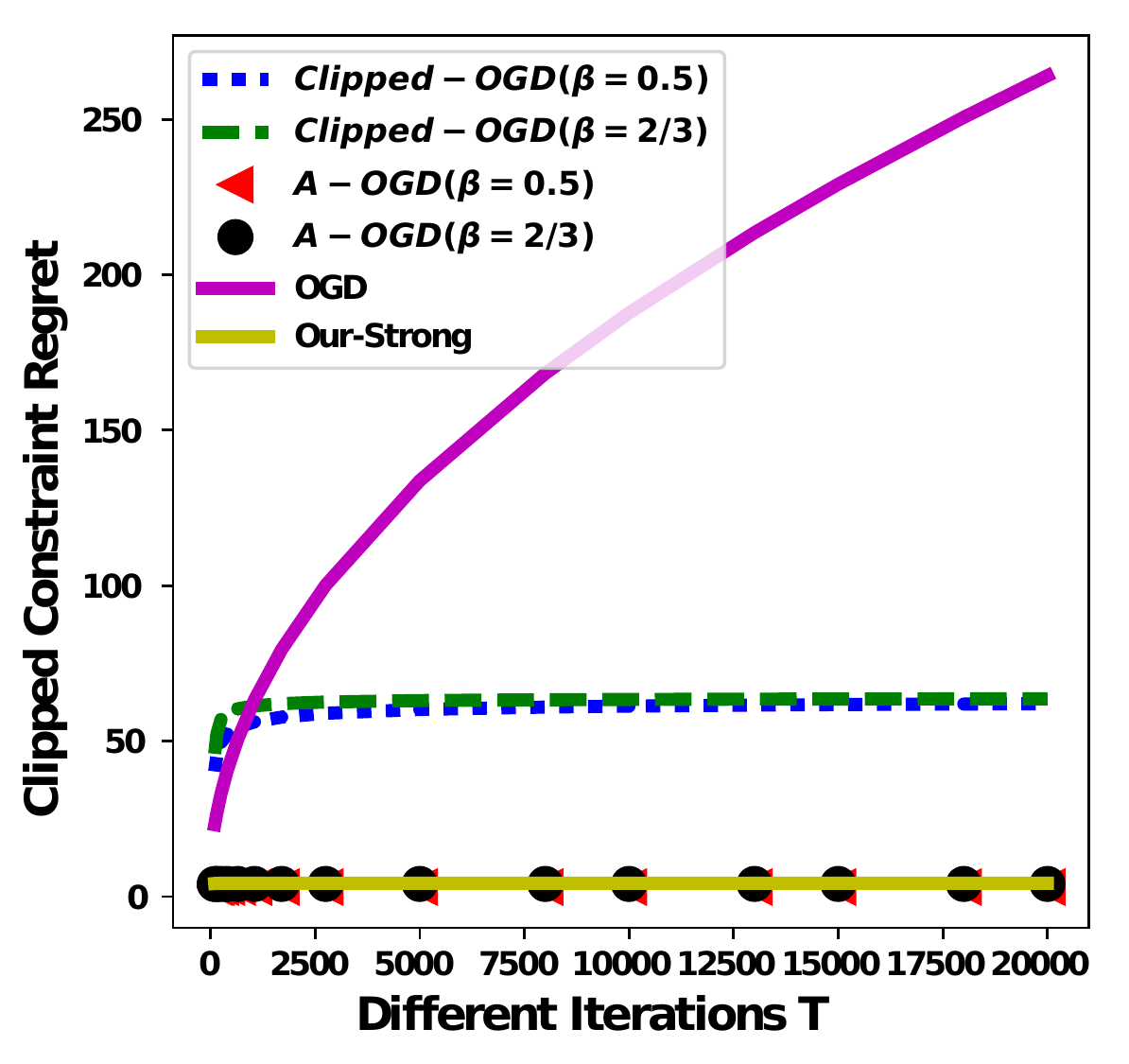}}
  \hspace{.3in}
  \subfigure[]{
    \label{fig::doubly_nonclip_con} 
    \includegraphics[height=3.cm]{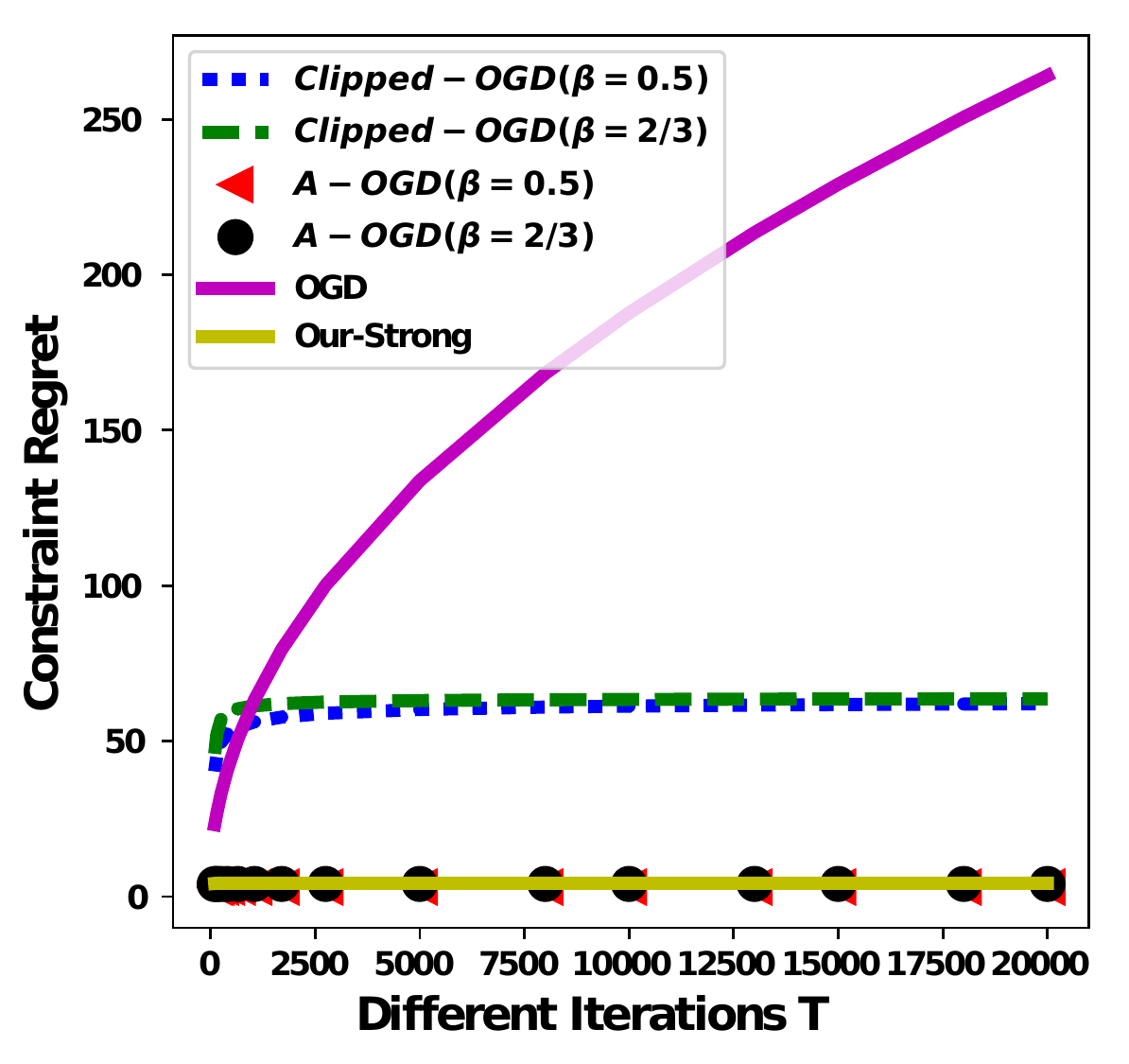}}
  \hspace{.3in}
  \subfigure[]{
    \label{fig::doubly_obj} 
    \includegraphics[height=3.cm]{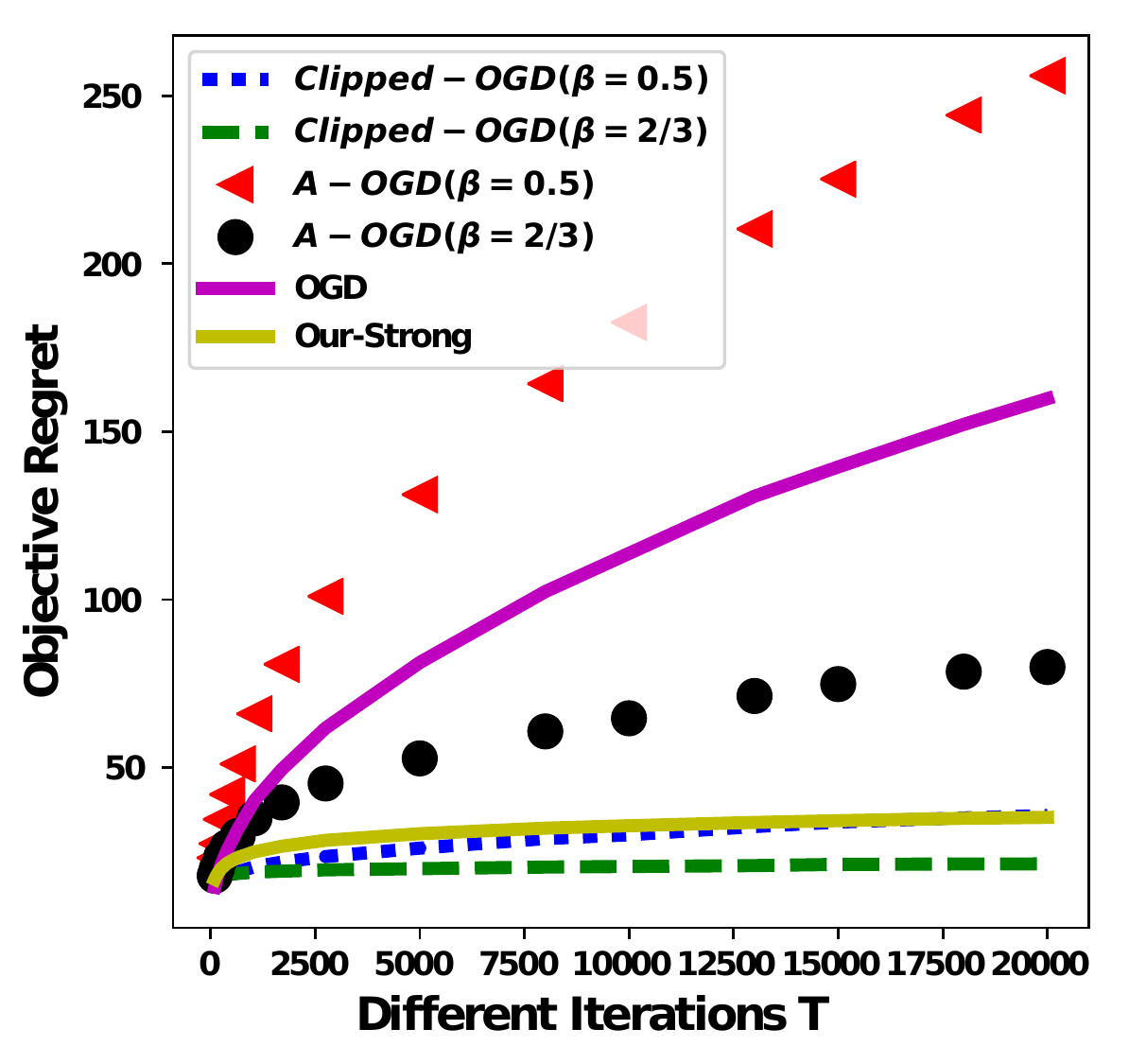}}
  \caption{\textbf{Doubly-Stochastic Matrices.} Fig.\ref{fig::doubly_clip_con}: Clipped Long-term Constraint Violation. 
           Fig.\ref{fig::doubly_nonclip_con}: Long-term Constraint Violation.
           Fig.\ref{fig::doubly_obj}: Cumulative Regret of the Loss function}
  \label{fig::doubly_obj_con} 
\vskip 0.in
\end{figure}

\begin{figure}
\vskip 0.0in
  \centering
  \subfigure[]{
    \label{fig::ed_demand} 
    \includegraphics[height=2.4cm]{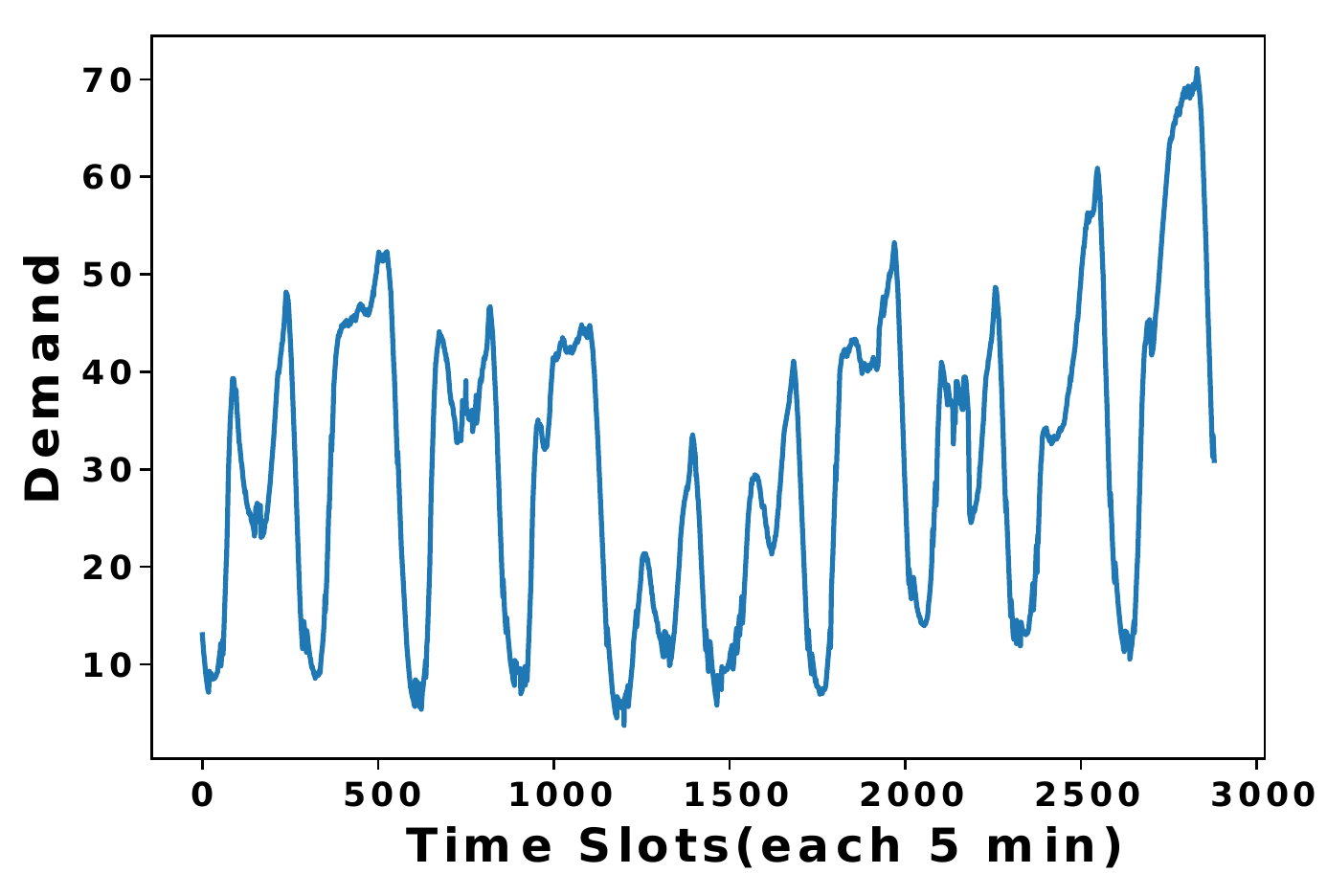}}
  \hspace{.1in}
  \subfigure[]{
    \label{fig::ed_con} 
    \includegraphics[height=3.1cm]{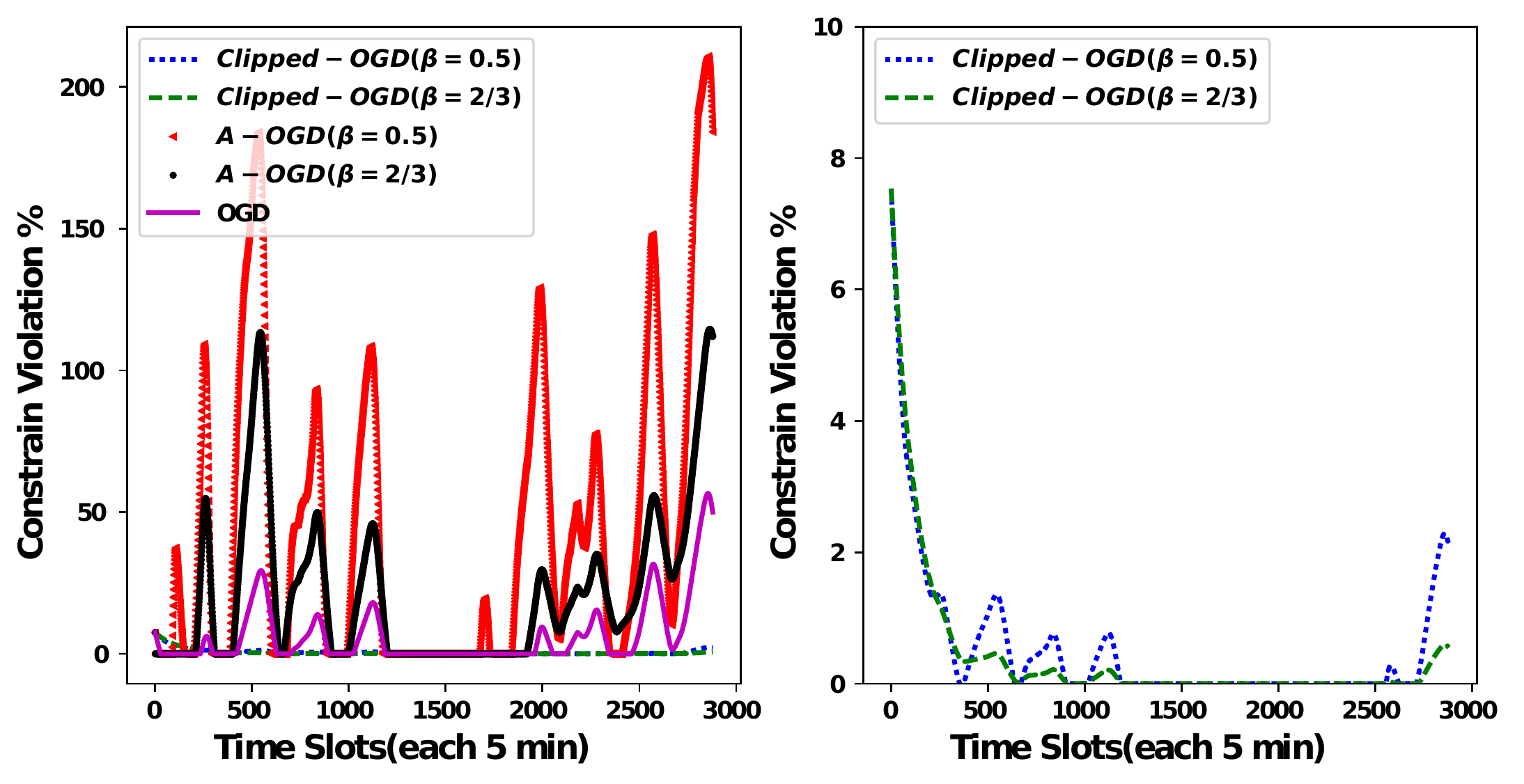}}
  \hspace{.1in}
  \subfigure[]{
    \label{fig::ed_obj} 
    \includegraphics[height=3.1cm]{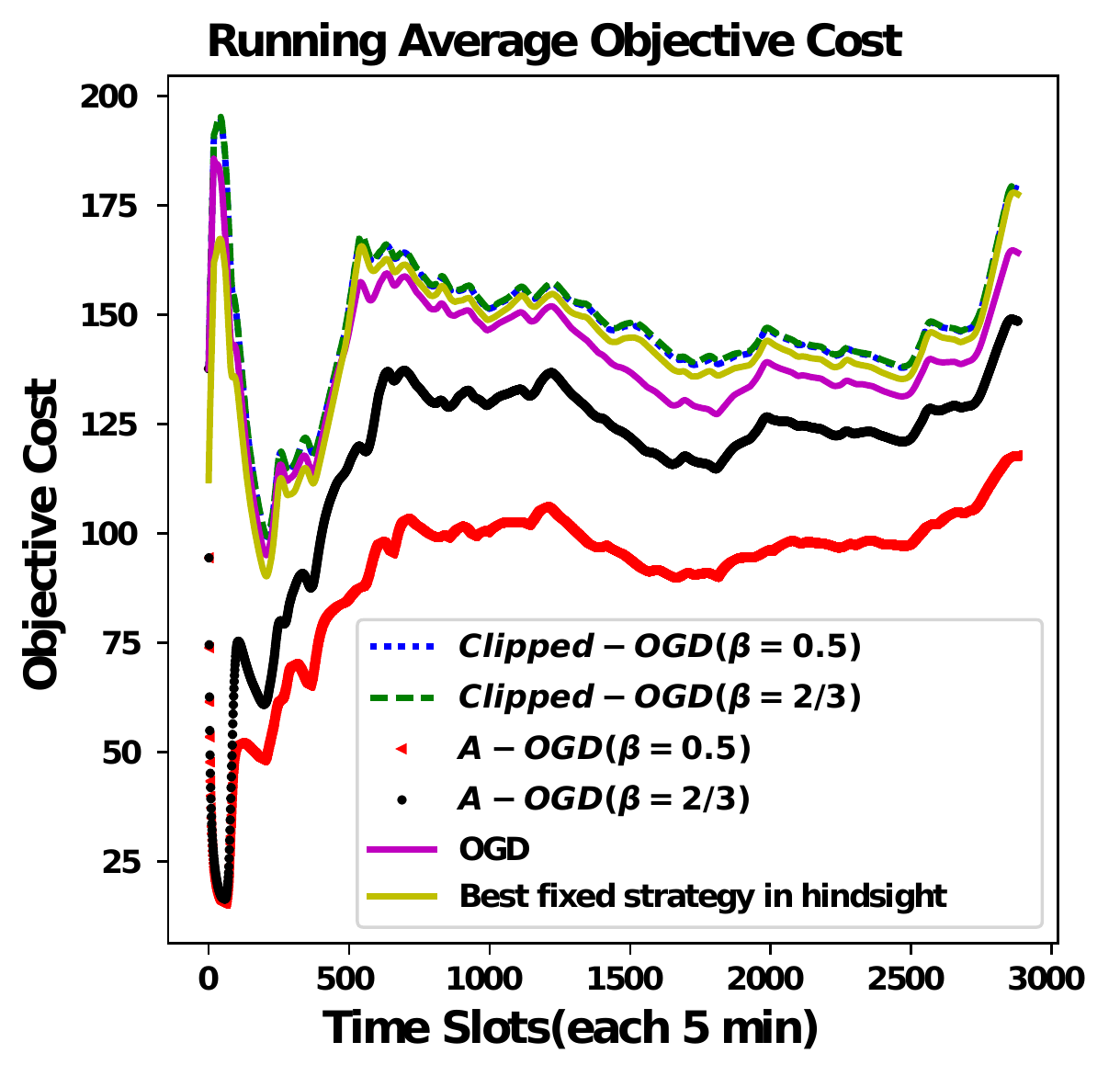}}
  \caption{\textbf{Economic Dispatch.}
           Fig.\ref{fig::ed_demand}: Power Demand Trajectory. 
           Fig.\ref{fig::ed_con}: Constraint Violation for each time
           step. All of the previous algorithms incurred substantial
           constraint violations. The figure on the right shows the
           violations of our algorithm, which are significantly smaller. 
           Fig.\ref{fig::ed_obj}: Running Average of the Objective Loss}
  \label{fig::ed_obj_con} 
\vskip 0.in
\end{figure}


In this section, we test the performance of the algorithms including
OGD \cite{mahdavi2012trading},
A-OGD \cite{jenatton2016adaptive}, Clipped-OGD (this paper),
and our proposed algorithm strongly convex case (Our-strong).
Throughout the experiments, our algorithm has the following fixed parameters:
$\alpha = 0.5$, $\sigma =\frac{(m+1)G^2}{2(1-\alpha)}$, $\eta = \frac{1}{T^\beta G\sqrt{R(m+1)}}$.
In order to better show the result of the constraint violation trajectories,
we aggregate all the constraints as a single one
by using $g(x_t) = \max_{i\in \{1,...,m\}}g_i(x_t)$ as done in \cite{mahdavi2012trading}.

\subsection{A Toy Experiment}

For illustration purposes, we solve the following 2-D toy experiment
with $x = [x_1,x_2]^T$:
\begin{equation}
\begin{array}{ll}
\min \sum\limits_{t=1}^T c_t^Tx,
&s.t.\left|x_1\right|+\left|x_2\right|-1\le 0.
\end{array}
\end{equation}
where the constraint is the $\ell_1$-norm constraint.
The vector $c_t$ is generated from a uniform
random vector over $[0,1.2]\times [0,1]$ which is rescaled to have
norm $1$. This leads to slightly average cost on the on the first coordinate.
The offline solutions for different $T$ are obtained by CVXPY \cite{cvxpy}.



All algorithms are run up to $T = 20000$ and are averaged over 10 random sequences of $\{c_t\}_{t=1}^T$.
Since the main goal here is to compare the variables' trajectories generated by different algorithms,
the results for different $T$ are in the supplementary material for space purposes.
Fig.\ref{fig::toy_traj} shows these trajectories for one realization with $T = 8000$.
The blue star is the optimal point's position. 

From Fig.\ref{fig::toy_traj} we can see that the trajectories generated by Clipped-OGD
follows the boundary very tightly until reaching the optimal point.
This can be explained by the Lemma \ref{lem:bound_step} which shows that 
the constraint violation for single step is also upper bounded.
For the OGD, the trajectory oscillates widely around the boundary of
the true constraint. 
For the A-OGD, its trajectory in Fig.\ref{fig::toy_traj} violates the
constraint most of the time,
and this violation actually contributes to the lower objective regret shown in the supplementary material.


\subsection{Doubly-Stochastic Matrices}

We also test the algorithms for approximation by doubly-stochastic matrices,
as in \cite{jenatton2016adaptive}:
\begin{equation}
\begin{array}{llll}
\min \sum\limits_{t=1}^T \frac{1}{2}\left\|Y_t-X\right\|_F^2
&s.t.\quad X\textbf{1} = \textbf{1},
     & X^T\textbf{1} = \textbf{1},
     & X_{ij}\ge 0.
\end{array}
\end{equation}
where $X\in \mathbb{R}^{d \times d}$ is the matrix variable, $\textbf{1}$ is the vector whose elements are all 1,
and matrix $Y_t$ is the permutation matrix which is randomly generated. 

After changing the equality constraints into inequality ones
(e.g.,$X\textbf{1} = \textbf{1}$ into $X\textbf{1} \ge \textbf{1}$ and $X\textbf{1} \le \textbf{1}$),
we run the algorithms with different T up to $T = 20000$ for 10 different random sequences of $\{Y_t\}_{t=1}^T$.
Since the objective function $f_t(x)$ is strongly convex
with parameter $H_1=1$, we also include our designed strongly convex algorithm as another comparison.
The offline optimal solutions are obtained by CVXPY \cite{cvxpy}.

The mean results for both constraint violation and objective regret
are shown in Fig.\ref{fig::doubly_obj_con}.
From the result we can see that,
for our designed strongly convex algorithm Our-Strong, its result is around the best ones in
not only the clipped constraint violation, 
but the objective regret. 
For our most-balanced convex case algorithm Clipped-OGD with 
$\beta = 0.5$, although its clipped constraint violation is relatively bigger than A-OGD,
it also becomes quite flat quickly, which means the algorithm quickly converges to a feasible solution.


\subsection{Economic Dispatch in Power Systems}
This example is adapted from \cite{li2018online} and \cite{senthil2010economic},
which considers the problem of power dispatch. That is, at each time step $t$,
we try to minimize the power generation cost $c_i(x_{t,i})$ for each generator $i$ while maintaining the
power balance $\sum\limits_{i=1}^n x_{t,i} = d_t$, where $d_t$ is the power demand at time $t$.
Also, each power generator produces an emission level $E_i(x_{t,i})$.
To bound the emissions, we impose the constraint $\sum\limits_{i=1}^n E_i(x_{t,i})\le E_{max}$. 
In addition to requiring this constraint to be satisfied on
average, 
we also require bounded constraint violations at each timestep. The
problem is formally stated as:
\begin{equation}
\begin{array}{lll}
\min \sum\limits_{t=1}^T \Big(\sum\limits_{i=1}^n c_i(x_{t,i})+\xi (\sum\limits_{i=1}^n x_{t,i}-d_t)^2 \Big),
&s.t. \quad \sum\limits_{i=1}^n E_i(t,i) \le E_{max},
     & 0\le x_{t,i} \le x_{i,max}.
\end{array}
\end{equation}
where the second constraint is from the fact that each generator has the power generation limit.

In this example, we use three generators. We define the cost and
emission functions according to \cite{senthil2010economic} and
\cite{li2018online} as $c_i(x_{t,i}) = 0.5a_i x_{t,i}^2+b_i x_{t,i}$,
and $E_i = d_i x_{t,i}^2+e_ix_{t,i}$, respectively.
The parameters are: $a_1 = 0.2, a_2 = 0.12, a_3 = 0.14$, $b_1 = 1.5, b_2 = 1, b_3 = 0.6$, $d_1 = 0.26, d_2 = 0.38, d_3 = 0.37$,
$E_{max} = 100$, $\xi = 0.5$, and $x_{1,max} = 20, x_{2,max} = 15, x_{3,max} = 18$.
The demand $d_t$ is adapted from real-world 5-minute interval demand data between 04/24/2018 and 05/03/2018
\footnote{https://www.iso-ne.com/isoexpress/web/reports/load-and-demand}, which is shown in Fig.\ref{fig::ed_demand}.
The offline optimal solution or best fixed strategy in hindsight is
obtained by an implementation of SAGA \cite{defazio2014saga}. 
The constraint violation for each time step is shown in Fig.\ref{fig::ed_con},
and the running average objective cost is shown in Fig.\ref{fig::ed_obj}.
From these results we can see that our algorithm has very small
constraint violation for each time step, 
which is desired by the requirement. Furthermore, our objective costs
are very close to the best fixed strategy.

\section{Conclusion}
In this paper, we propose two algorithms for OCO with both convex and strongly convex objective functions.
By applying different update strategies that utilize a modified augmented Lagrangian function,
they can solve OCO with a squared/clipped long-term constraints requirement.
The algorithm for general convex case provides the useful bounds for
both  the long-term constraint violation and the constraint violation
at each timestep.
Furthermore, the bounds for the strongly convex case is an improvement compared with the previous efforts in the literature.
Experiments show that our algorithms can follow the constraint boundary tightly and
have relatively smaller clipped long-term constraint violation with reasonably low objective regret.
It would be useful if future work could explore the noisy versions of the constraints and obtain the similar upper bounds.

\section*{Acknowledgments}

Thanks to Tianyi Chen for valuable discussions about algorithm's properties. 

\bibliography{OCO}

\bibliographystyle{plainnat}

\newpage

\begin{center}
\textbf{\large Supplemental Materials}
\end{center}

\appendix

The supplemental material contains proofs of the main results of the
paper along with supporting results.

\section{Toy Example Results}

The results including different $T$ up to $20000$ are shown in Fig.\ref{fig::toy_obj_con}, 
whose results are averaged over 10 random sequences of $\{c_t\}_{t=1}^T$.
Since the standard deviations are small, we only plot the mean results. 

From Fig.\ref{fig::toy_traj} we can see that the trajectories generated by $Clipped-OGD$
follows the boundary very tightly until reaching the optimal point.
which is also reflected by the Fig.\ref{fig::toy_clip_con} of the clipped long-term constraint violation.
For the $OGD$, its trajectory oscillates a lot around the boundary of
the actual constraint. 
And if we examine the clipped and non-clipped constraint violation in Fig.\ref{fig::toy_obj_con},
we find that although the clipped constraint violation is very high, its non-clipped one is very small.
This verifies the statement we make in the beginning that the big constraint violation at one time step is canceled out 
by the strictly feasible constraint at the other time step.
For the $A-OGD$, its trajectory in Fig.\ref{fig::toy_traj} violates the constraint most of the time,
and this violation actually contributes to the lower objective regret shown in Fig.\ref{fig::toy_obj_con}.

\begin{figure}
\vskip 0.0in
  \centering
  \subfigure[]{
    \label{fig::toy_clip_con} 
    \includegraphics[height=3.9cm]{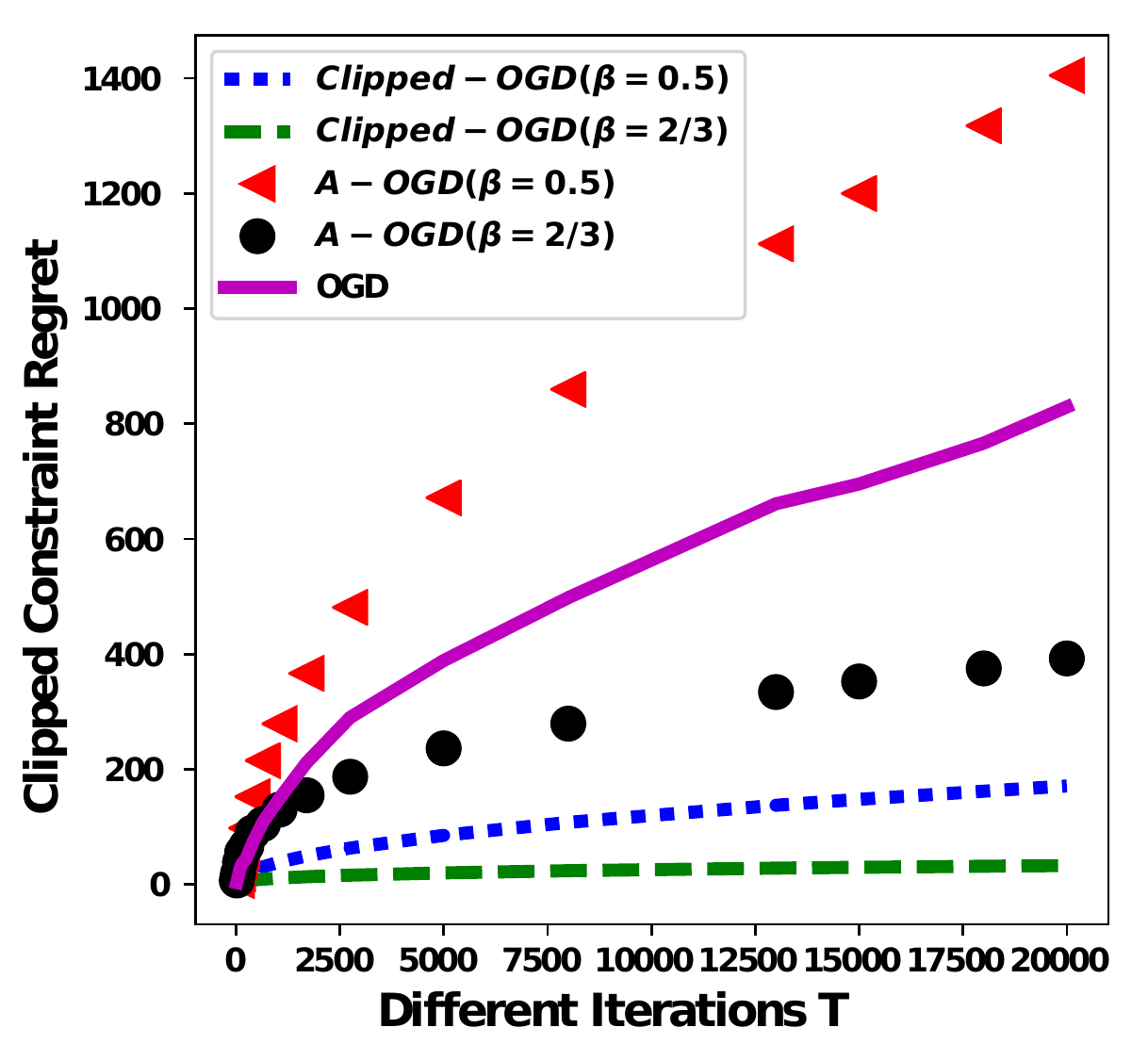}}
  \hspace{.1in}
  \subfigure[]{
    \label{fig::toy_nonclip_con} 
    \includegraphics[height=3.9cm]{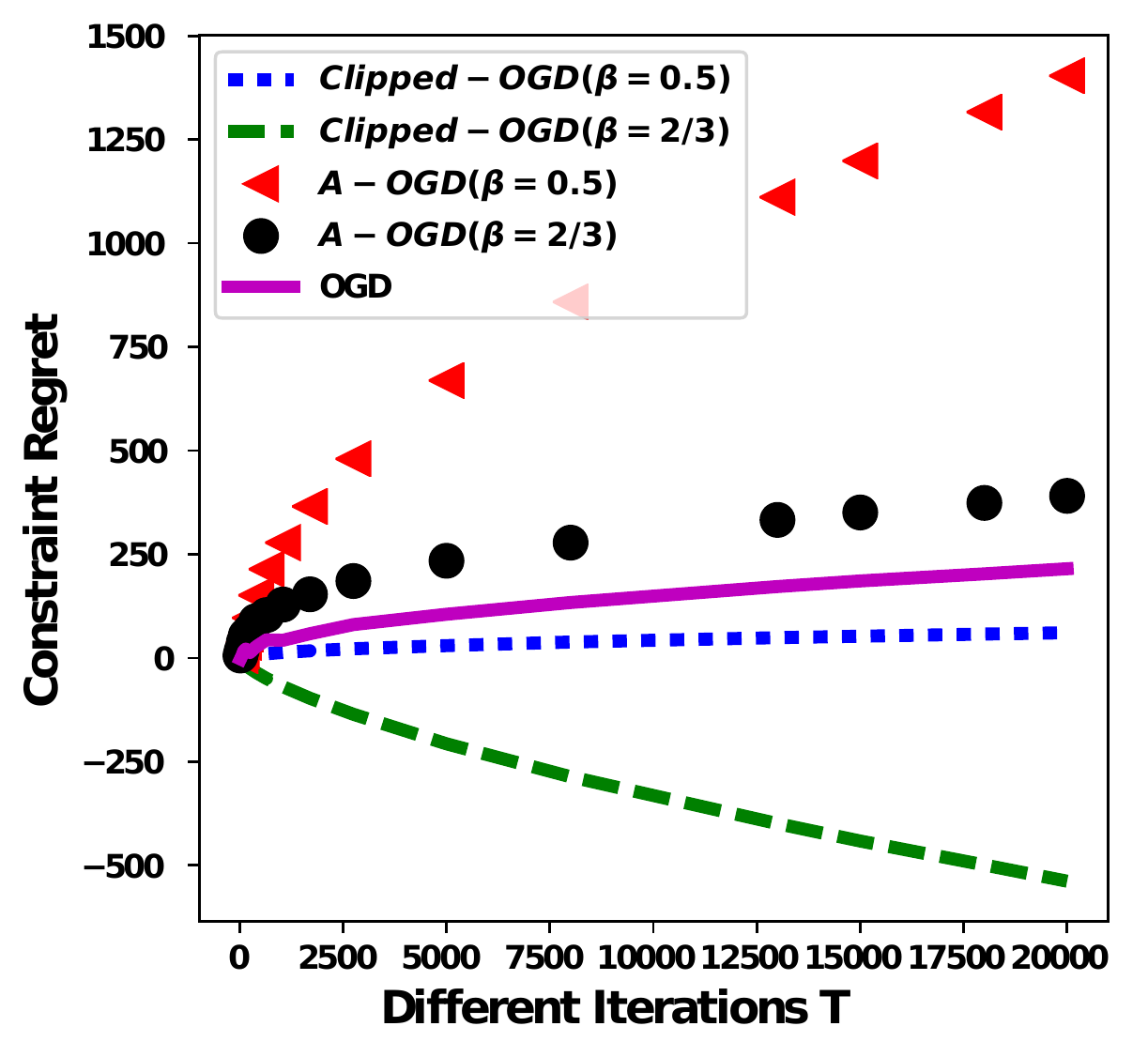}}
  \hspace{.1in}
  \subfigure[]{
    \label{fig::toy_obj} 
    \includegraphics[height=3.9cm]{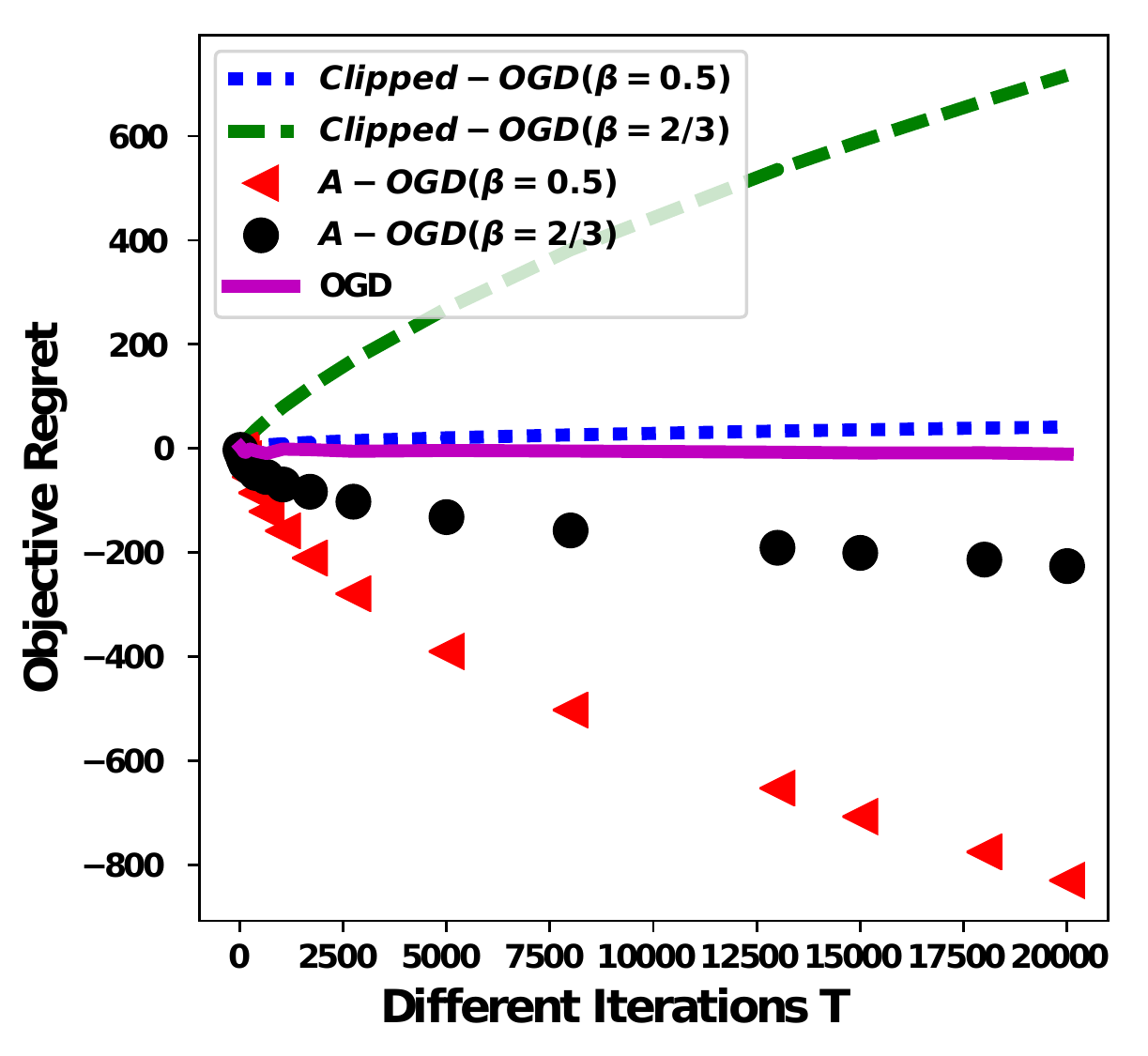}}
  \caption{Toy Example Results: 
           Fig.\ref{fig::toy_clip_con}: Clipped Long-term Constraint Violation. 
           Fig.\ref{fig::toy_nonclip_con}: Long-term Constraint Violation.
           Fig.\ref{fig::toy_obj}: Cumulative Regret of the Loss function}
  \label{fig::toy_obj_con} 
\vskip 0.in
\end{figure}

\section{Proof of Theorem 1}
Before proving Theorem~\ref{thm::sumOfSquareLongterm}, we need the
following preliminary result. 

\begin{lemma}
  \label{lem:sumOfLagfunction}
  {\it
    For the sequence of $x_t$, $\lambda_t$ obtained from Algorithm \ref{alg::alg1} and $\forall x\in\mathcal{B}$, we can prove the following inequality:
    \begin{equation}
    \begin{array}{rl}
    \sum\limits_{t=1}^T[\mathcal{L}_t(x_t,\lambda_t)-\mathcal{L}_t(x,\lambda_t)]&\le 
    \frac{R^2}{2\eta}+\frac{\eta T}{2}(m+1)G^2 \\
    &+\frac{\eta}{2}(m+1)G^2\sum\limits_{t=1}^T\left\|\lambda_t\right\|^2
    \end{array}
    \end{equation}

  }
\end{lemma}

\begin{proof}
First, $\mathcal{L}_t(x,\lambda)$ is convex in $x$. Then for any $x\in\mathcal{B}$,
we have the following inequality:
\begin{equation}
\mathcal{L}_t(x_t,\lambda_t)-\mathcal{L}_t(x,\lambda_t) \le (x_t-x)^T\partial_x\mathcal{L}_t(x_t,\lambda_t)
\end{equation}
Using the non-expansive property of the projection operator and the update rule for $x_{t+1}$ in Algorithm \ref{alg::alg1}, 
we have
\begin{equation}
\label{eq::x_projection_inequality}
\begin{array}{ll}
\left\|x-x_{t+1}\right\|^2& \le \left\|x-(x_t-\eta\partial_x \mathcal{L}_t(x_t,\lambda_t))\right\|^2 \\
                          & = \left\|x-x_t\right\|^2-2\eta(x_t-x)^T\partial_x\mathcal{L}_t(x_t,\lambda_t) \\
                          & \quad +\eta^2\left\|\partial_x\mathcal{L}_t(x_t,\lambda_t)\right\|^2  
\end{array}
\end{equation}
Then we have
\begin{equation}
\label{eq::L_t_diff_inequality}
\small
\begin{array}{rl}
\mathcal{L}_t(x_t,\lambda_t)-\mathcal{L}_t(x,\lambda_t) & \le \frac{1}{2\eta}\Big(\left\|x-x_t\right\|^2-\left\|x-x_{t+1}\right\|^2\Big) \\
                                                        & \quad +\frac{\eta}{2}\left\|\partial_x\mathcal{L}_t(x_t,\lambda_t)\right\|^2
\end{array}
\end{equation}

Furthermore, for $\left\|\partial_x\mathcal{L}_t(x_t,\lambda_t)\right\|^2$, we have
\begin{equation}
\label{eq::grad_lag_inequality}
\begin{array}{rl}
\left\|\partial_x\mathcal{L}_t(x_t,\lambda_t)\right\|^2 &= \left\|\partial_xf_t(x_t) + \sum\limits_{i=1}^m\lambda_t^i\partial_x([g_i(x_t)]_+)\right\|^2 \\
                                                        & \le (m+1)G^2(1+\left\|\lambda_t\right\|^2)

\end{array}
\end{equation}
where the last inequality is from the inequality that $(y_1+y_2+...+y_n)^2\le n(y_1^2+y_2^2+...+y_n^2)$, 
and both $\left\|\partial_xf_t(x_t)\right|$ and $\left\|\partial_x([g_i(x_t)]_+)\right\|$ are less than or equal to $G$ by the definition.

Then we have 
\begin{equation}
\small
\begin{array}{rl}
\mathcal{L}_t(x_t,\lambda_t)-\mathcal{L}_t(x,\lambda_t)& \le \frac{1}{2\eta}\Big(\left\|x-x_t\right\|^2-\left\|x-x_{t+1}\right\|^2\Big) \\
                                                       & \quad +\frac{\eta}{2}(m+1)G^2(1+\left\|\lambda_t\right\|^2)

\end{array}
\end{equation}

Since $x_1$ is in the center of $\mathcal{B}$, we can assume $x_1 = 0$ without loss of generality. 
If we sum the $\mathcal{L}_t(x_t,\lambda_t)-\mathcal{L}_t(x,\lambda_t)$ from 1 to $T$, we have
\begin{equation}
\scriptsize
\begin{array}{ll}
\sum\limits_{t=1}^T[\mathcal{L}_t(x_t,\lambda_t)-\mathcal{L}_t(x,\lambda_t)]& \le \frac{1}{2\eta}\Big(\left\|x-x_1\right\|^2-\left\|x-x_{T+1}\right\|^2\Big) \\
                                                & \quad +\frac{\eta T}{2}(m+1)G^2 \\
                                                & \quad +\frac{\eta}{2}(m+1)G^2\sum\limits_{t=1}^T\left\|\lambda_t\right\|^2 \\
                                                & \le \frac{R^2}{2\eta}+\frac{\eta T}{2}(m+1)G^2 \\
                                                & \quad +\frac{\eta}{2}(m+1)G^2\sum\limits_{t=1}^T\left\|\lambda_t\right\|^2

\end{array}
\end{equation}
where the last inequality follows from the fact that $x_1 = 0$ and $\left\|x\right\|^2\le R^2$.
\end{proof}

Now we are ready to prove the main theorem.

\begin{proof}[Proof of Theorem~\ref{thm::sumOfSquareLongterm}]
From Lemma \ref{lem:sumOfLagfunction}, we have 
\begin{equation}
\begin{array}{rl}
  \sum\limits_{t=1}^T[\mathcal{L}_t(x_t,\lambda_t)-\mathcal{L}_t(x,\lambda_t)]&\le 
  \frac{R^2}{2\eta}+\frac{\eta T}{2}(m+1)G^2 \\
  &+\frac{\eta}{2}(m+1)G^2\sum\limits_{t=1}^T\left\|\lambda_t\right\|^2
\end{array}
\end{equation}

If we expand the terms in the left-hand side and move the last term in right-hand side to the left,
we have
\begin{equation}
\small
\begin{array}{l}
\sum\limits_{t=1}^T\Big(f_t(x_t) - f_t(x)\Big) +
\sum\limits_{t=1}^T\sum\limits_{i=1}^m\Big(\lambda_t^i[g_i(x_t)]_+ -
\lambda_t^i[g_i(x)]_+\Big) \\
-\frac{\eta}{2}(m+1)G^2\sum\limits_{t=1}^T\left\|\lambda_t\right\|^2
\le \frac{R^2}{2\eta}+\frac{\eta T}{2}(m+1)G^2
\end{array}
\end{equation}

If we set $x = x^*$ to have $[g_i(x^*)]_+ = 0$ 
and plug in the expression $\lambda_t = \frac{[g(x_t)]_+}{\sigma\eta}$,
we have

\begin{equation}
\label{eq::UpperBoundOfsumOfObjAndLongTermConstrain}
\small
\begin{array}{rl}
\sum\limits_{t=1}^T\Big(f_t(x_t) - f_t(x^*)\Big) & +
\sum\limits_{i=1}^m\sum\limits_{t=1}^T\frac{([g_i(x_t)]_+)^2}{\sigma\eta}\Big(1-\frac{(m+1)G^2}{2\sigma}\Big) \\
&\le \frac{R^2}{2\eta}+\frac{\eta T}{2}(m+1)G^2

\end{array}
\end{equation}

If we plug in the expression for $\sigma$ and $\eta$, we have 

\begin{equation}
\small
\begin{array}{rl}
\sum\limits_{t=1}^T\Big(f_t(x_t) - f_t(x^*)\Big) & +
\sum\limits_{i=1}^m\sum\limits_{t=1}^T\frac{([g_i(x_t)]_+)^2}{\sigma\eta}\alpha \\
&\le O(\sqrt{T})

\end{array}
\end{equation}

Because $\frac{([g_i(x_t)]_+)^2}{\sigma\eta}\alpha\ge 0$, we have
\begin{equation}
\sum\limits_{t=1}^T\Big(f_t(x_t) - f_t(x^*)\Big) \le O(\sqrt{T})
\end{equation} 

Furthermore, we have $\sum\limits_{t=1}^T\Big(f_t(x_t) - f_t(x^*)\Big)\ge -FT$
according to the assumption.
Then we have

\begin{equation}
\begin{array}{l}
\sum\limits_{i=1}^m\sum\limits_{t=1}^T\Big([g_i(x_t)]_+\Big)^2  \le \frac{\sigma\eta}{\alpha}(O(\sqrt{T})+FT)\\
    = \frac{\sigma}{\alpha}(O(\sqrt{T})+FT)O(\frac{1}{\sqrt{T}}) = O(\sqrt{T})
\end{array}
\end{equation}

Because $\Big([g_i(x_t)]_+\Big)^2 \ge 0$, we have 
\begin{equation}
\sum\limits_{t=1}^T\Big([g_i(x_t)]_+\Big)^2 \le O(\sqrt{T}), \forall i \in \{1,2,...,m\}
\end{equation}
\end{proof}

\section{Proof of Lemma 1}

\begin{proof}
Recall that the update for $x_{t+1}$ is
\begin{equation}
\label{eq::update_x_lemm1}
x_{t+1} = \Pi_{\mathcal{B}}\Big(x_t-\eta \partial_x f_t(x_t) - \frac{[g(x_t)]_+}{\sigma} \partial_x ([g(x_t)]_+)\Big)
\end{equation}
Let $y_t=x_t-\eta \partial_x f_t(x_t) - \frac{[g(x_t)]_+}{\sigma} \partial_x ([g(x_t)]_+)$.

We first need to show that $g(x_{t+1})\le g(y_t)$.
Without loss of generality, let us assume that $y_t$ is not in the set $\mathcal{B}$.
From convexity we have $g(y_t)\ge g(x_{t+1})+\nabla_x g(x_{t+1})^T(y_t-x_{t+1})$.
From non-expansiveness of the projection operator, we have that 
$(y_t-x_{t+1})^T(x-x_{t+1})\le 0$ for $x\in \mathcal{B}$. 
Let $x = x_{t+1} - \epsilon_0\nabla_x g(x_{t+1})$ with $\epsilon_0$ small enough to make $x\in \mathcal{B}$.
We have $-\epsilon_0 (y_t-x_{t+1})^T\nabla_x g(x_{t+1}) \le 0$. 
Then we have $g(x_{t+1})\le g(y_t)$.

As a result, if $g(y_t)$ is upper bounded,  then so is $g(x_{t+1})$,
where $x_{t+1} = \Pi_{\mathcal{B}}(y_t)$. 
If $T$ is large enough,
$\eta\left\|\partial_x f_t(x_t)\right\|$ would be very small. 
Thus, we can use $0$-order Taylor expansion for differentiable $g(x)$ as below:
\begin{equation}
  \label{eq:simpleTaylor}
\begin{array}{ll}
g(y_t) = &g\Big(x_t-\eta\partial_x
           f_t(x_t)-\frac{[g(x_t)]_+}{\sigma} \partial_x
           ([g(x_t)]_+)\Big)
  \\
  &\le g\Big(x_t-\frac{[g(x_t)]_+}{\sigma} \partial_x
    ([g(x_t)]_+)\Big)+C \eta 
\end{array}
\end{equation}
where $C$ is a constant determined by the Taylor expansion remainder, as well as
the bound $\|\partial_x [g(x_t)]_+\| \|\partial_X f(x_t)\| \le G^2$.



Set $\epsilon =(2 C \sigma R^2\eta)^{1/3}=O(\frac{1}{T^{1/6}})$. We will show that if $g(x_t) <
\epsilon$, then $g(x_{t+1}) \le \epsilon + O(1/\sqrt{T}) =
O(\frac{1}{T^{1/6}})$. We will also show that if $g(x_t)\ge \epsilon$,
then $g(x_{t+1})\le g(x_t)$. It follows then by induction that if
$g(x_1)<\epsilon$, then $g(x_t) = O(\frac{1}{T^{1/6}})$ for all
$t$. We prove these
inequalities in three cases. Since $g(x_{t+1})\le g(y_t)$, it suffices
to bound $g(y_t)$. 

\textbf{Case 1: $g(x_t) \le 0$.} In this case, the inequality for
$g(y_t)$, \eqref{eq:simpleTaylor}, becomes 
\begin{equation}
g(y_t) \le g(x_t)+C\eta \le C\eta = O(\frac{1}{\sqrt{T}})
\end{equation}

\textbf{Case 2: $0<g(x_t) < \epsilon$.} Since $[g(x_t)]_+=g(x_t)$, the bound on $g(y_t)$ becomes
\begin{equation}
  \label{eq:smallPosBound}
g(y_t) \le g\Big(x_t-\frac{g(x_t)}{\sigma} \nabla_x g(x_t)\Big) + C\eta
\end{equation}
We will bound  the right using standard methods from gradient descent proofs.
Since $g$ is convex and $\nabla_x g(x)$ has Lipschitz constant, $L$,
we have the inequality:
\begin{equation}
  \label{eq:convexLipschitzBound}
  g(y)\le g(x) + \nabla_x g(x)^T (y-x)+\frac{L}{2}\|y-x\|^2  
\end{equation}
for all $x$ and $y$ \cite{nesterov2013introductory}.

Recall that
$\epsilon = O(\frac{1}{T^{1/6}})$. Assume that $T$ is sufficiently
large so that $\frac{Lg(x_t)}{2\sigma}<\frac{L\epsilon}{2\sigma} < 1$. 
Applying
\eqref{eq:convexLipschitzBound} with $x=x_t$ and
$y=x_t-\frac{g(x_t)}{\sigma}\nabla_x g(x_t)$ gives
\begin{align}
  g(y_t) & \le  g\Big(x_t-\frac{[g(x_t)]_+}{\sigma} \partial_x
           (g(x_t))\Big) + C\eta \\ 
  \label{eq:specialLipschitzBound}
  &\le
                                                               g(x_t)
                                                               -
                                                               \frac{g(x_t)}{\sigma}(1-\frac{Lg(x_t)}{2\sigma})\left\|\nabla_x
                                                               g(x_t)
                                                               \right\|^2
  + C\eta \\
  &\le g(x_t) + C\eta = O(\frac{1}{T^{1/6}}). 
\end{align}
where the third bound follows since $1-\frac{Lg(x_t)}{2\sigma} >
0$.

\textbf{Case 3: $g(x_t)\ge \epsilon$. }
A case can arise such that $g(x_{t-1})<\epsilon$ but an additive term of order
$O(\frac{1}{T^{1/2}})$ leads to $\epsilon \le g(x_{t})\le \epsilon +
C\eta = O(\frac{1}{T^{1/6}})$. We will now show that no further
increases are possible by bounding the final two terms of
\eqref{eq:specialLipschitzBound} as
\begin{equation}
  \label{eq:equivBound}
                                         -
                                                               \frac{g(x_t)}{\sigma}(1-\frac{Lg(x_t)}{2\sigma})\left\|\nabla_x
                                                               g(x_t)
                                                               \right\|^2
  + C\eta \le 0 \iff C\eta \le  \frac{g(x_t)}{\sigma}(1-\frac{Lg(x_t)}{2\sigma})\left\|\nabla_x
                                                               g(x_t)
                                                               \right\|^2.
\end{equation}
Now, we
lower-bound the terms on the right of \eqref{eq:equivBound}. Since $\epsilon +
C\eta=O(\frac{1}{T^{1/6}})$, we have that for sufficiently large $T$,
$1-\frac{Lg(x_t)}{2\sigma}\ge 1-\frac{L(\epsilon+C\eta)}{2\sigma}\ge
\frac{1}{2}$.  Further note that by convexity, $g(0)\ge
g(x_t)-\nabla_x g(x_t)^T x_t$. Since we assume that $0$ is feasible,
we have that
\begin{equation*}
  \epsilon \le g(x_t) \le \nabla_x g(x_t)^T x_t \le \|\nabla_x
  g(x_t)\| \|x_t\| \le \|\nabla_x g(x_t)\| R.
\end{equation*}
The final inequality follows since $x_t\in\mathcal{B}$. Thus, we have
the following bound for the right of \eqref{eq:equivBound}:
\begin{equation*}
  \frac{g(x_t)}{\sigma}(1-\frac{Lg(x_t)}{2\sigma})\left\|\nabla_x
    g(x_t)\right\|^2
    \ge \frac{\epsilon^3}{2\sigma R^2} = C\eta.
  \end{equation*}
  The final equality follows by the definition of $\epsilon$. 
\end{proof}

\section{Proof of Theorem 2}

\begin{proof}
For the strongly convex case of $f_t(x)$ with strong convexity parameter equal to $H_1$,
we can also conclude that the modified augmented Lagrangian function in Eq.(\ref{eq::update_strongly_convex})
is also strongly convex w.r.t. $x$ with the strong convexity parameter $H\ge H_1$.
Then we have 
\begin{equation}
\label{inequal::strong_lag}
\begin{array}{rl}
\mathcal{L}_t(x^*,\lambda_t)-\mathcal{L}_t(x_t,\lambda_t) &\ge 
\partial_x\mathcal{L}_t(x_t)^T(x^*-x_t)\\
& + \frac{H_1}{2}\left\|x^*-x_t\right\|^2
\end{array}
\end{equation}

From concavity of $\mathcal{L}$ in terms of $\lambda$, we can have
\begin{equation}
\label{inequal::strong_lambda}
\mathcal{L}_t(x_t,\lambda) - \mathcal{L}_t(x_t,\lambda_t)\le (\lambda-\lambda_t)^T\nabla_{\lambda}\mathcal{L}_t(x_t,\lambda_t)
\end{equation}
Since $\lambda_t$ maximizes the augmented Lagrangian, we can see that the right hand side is $0$.

From Eq.(\ref{eq::x_projection_inequality}), we have
\begin{equation}
\label{inequal::strong_proj}
\begin{array}{rl}
\partial_x\mathcal{L}_t(x_t)^T(x_t-x^*) \le& \frac{1}{2\eta_t}\Big(\left\|x^*-x_t\right\|^2-\left\|x^*-x_{t+1}\right\|^2\Big)\\
& +\frac{\eta_t}{2}(m+1)G^2(1+\left\|\lambda_t\right\|^2)

\end{array}
\end{equation}

Multiply Eq.(\ref{inequal::strong_lag}) by $-1$ and add Eq.(\ref{inequal::strong_lambda}) together with Eq.(\ref{inequal::strong_proj}) plugging in:
\begin{equation}
\begin{array}{l}
\mathcal{L}_t(x_t,\lambda) - \mathcal{L}_t(x^*,\lambda_t) \le \frac{1}{2\eta_t}\Big(\left\|x^*-x_t\right\|^2-\left\|x^*-x_{t+1}\right\|^2\Big)\\
\quad +\frac{\eta_t}{2}(m+1)G^2(1+\left\|\lambda_t\right\|^2)-\frac{H_1}{2}\left\|x^*-x_t\right\|^2
\end{array}
\end{equation}

Let $b_t=\left\|x^*-x_t\right\|^2$, and plug in the expression for $\mathcal{L}_t$, we can get:
\begin{equation}
\begin{array}{l}
f_t(x_t) - f_t(x^*) +\lambda^T[g(x_t)]_+-\frac{\theta_t}{2}\left\|\lambda\right\|^2 \le \frac{1}{2\eta_t}(b_t-b_{t+1})\\
\quad -\frac{H_1}{2}b_t +\frac{(m+1)G^2}{2}\eta_t+\frac{(m+1)G^2}{2}\left\|\lambda_t\right\|^2(\eta_t-\frac{\theta_t}{(m+1)G^2})
\end{array}
\end{equation}

Plug in the expressions $\eta_t = \frac{1}{H_1(t+1)}$, $\theta_t = (m+1)G^2\eta_t$, and sum over $t=1$ to $T$:
\begin{equation}
\begin{array}{l}
\sum\limits_{t=1}^T\Big(f_t(x_t)-f_t(x^*)\Big) + \lambda^T\Big(\sum\limits_{t=1}^T[g(x_t)]_+\Big)
-\frac{\left\|\lambda\right\|^2}{2}\sum\limits_{t=1}^T\theta_t \\
\le \underbrace{\frac{1}{2}\sum\limits_{t=1}^T\Big(\frac{b_t-b_{t+1}}{\eta_t}-\frac{H_1}{2}b_t\Big)}_A + \underbrace{\frac{(m+1)G^2}{2}\sum\limits_{t=1}^T\eta_t}_B
\end{array}
\end{equation}

For the expression of $A$, we have:
\begin{equation}
\begin{array}{l}
A = \frac{1}{2}\Big[\frac{b_1}{\eta_1}+\sum\limits_{t=2}^Tb_t(\frac{1}{\eta_t}-\frac{1}{\eta_{t-1}}-H_1)-\frac{b_{T+1}}{\eta_T}-H_1b_1\Big] \\
\quad \le b_1H_1
\end{array}
\end{equation}

For the expression of $B$, with the expression of $\eta_t$ and the inequality relation between sum and integral, we have:
\begin{equation}
B\le \frac{(m+1)G^2}{2H_1}\log(T)
\end{equation}

Thus, we have:
\begin{equation}
\begin{array}{l}
\sum\limits_{t=1}^T\Big(f_t(x_t)-f_t(x^*)\Big) + \lambda^T\Big(\sum\limits_{t=1}^T[g(x_t)]_+\Big)
-\frac{\left\|\lambda\right\|^2}{2}\sum\limits_{t=1}^T\theta_t \\
\le O(\log(T))
\end{array}
\end{equation}

If we set $\lambda = \frac{\sum\limits_{t=1}^T[g(x_t)]_+}{\sum\limits_{t=1}^T\theta_t}$, 
and due to non-negativity of $\frac{\Big\|\sum\limits_{t=1}^T[g(x_t)]_+\Big\|^2}{2\sum\limits_{t=1}^T\theta_t}$,
we can have
\begin{equation}
\sum\limits_{t=1}^T\Big(f_t(x_t)-f_t(x^*)\Big) \le O(\log(T))
\end{equation}

Furthermore, we have $\sum\limits_{t=1}^T\Big(f_t(x_t) - f_t(x^*)\Big)\ge -FT$
according to the assumption.
Then we have
\begin{equation}
\frac{\Big\|\sum\limits_{t=1}^T[g(x_t)]_+\Big\|^2}{2\sum\limits_{t=1}^T\theta_t} \le O(\log(T)) + FT
\end{equation}

Because $\sum\limits_{t=1}^T\theta_t\le (m+1)G^2\log(T)/H_1$, we have:
\begin{equation}
\big\|\sum\limits_{t=1}^T[g(x_t)]_+\big\| \le O(\sqrt{\log(T)T})
\end{equation}

\end{proof}

\section{Proof of the Propositions}

Now we give the proofs for all the remaining Propositions.

\begin{proof}[Proof of the Proposition \ref{prop::bound_step_max}]
From the construction of $\bar{g}(x)$, we have the 
$\bar{g}(x)\ge \max\limits_{i}g_i(x)$. Thus, if we can upper bound the $\bar{g}(x)$,
$g_i(x)$ will automatically be upper bounded.
In order to use Lemma \ref{lem:bound_step}, we need to make sure
the following conditions are satisfied:
\begin{itemize}
\item $\bar{g}(x)$ is convex and differentiable. 
\item $\left\|\nabla_x \bar{g}(x)\right\|$ is upper bounded. 
\item $\left\|\nabla_x^{\prime\prime} \bar{g}(x)\right\|_2$ is upper bounded.
\end{itemize}
The first condition is satisfied due to the formula of $\bar{g}(x)$.
To examine the second one, we have 
\begin{equation}
\nabla_x \bar{g}(x) = \frac{1}{\sum\limits_{i=1}^m \exp g_i(x)}\Bigg[\sum\limits_{i=1}^m \exp g_i(x) \nabla_x g_i(x)\Bigg]
\end{equation}
\begin{equation}
\begin{array}{rl}
\left\|\nabla_x \bar{g}(x)\right\|^2 &= \frac{1}{\big(\sum\limits_{i=1}^m \exp g_i(x)\big)^2}
\left\|\sum\limits_{i=1}^m \exp g_i(x) \nabla_x g_i(x)\right\|^2 \\
& \le \frac{m\sum\limits_{i=1}^m (\exp g_i(x))^2\left\|\nabla_x g_i(x)\right\|^2 }{\big(\sum\limits_{i=1}^m \exp g_i(x)\big)^2}  \\
& \le m G^2
\end{array}
\end{equation}
Thus, $\left\|\nabla_x \bar{g}(x)\right\| \le \sqrt{m}G$ and the second condition is satisfied.

For $\left\|\nabla_x^{\prime\prime} \bar{g}(x)\right\|_2$, we have 
\begin{equation}
\begin{array}{ll}
\nabla_x^{\prime\prime} \bar{g}(x) =& \underbrace{\frac{1}{\sum\limits_{i=1}^m \exp g_i(x)}
\Bigg[\sum\limits_{i=1}^m \exp g_i(x) \nabla_x^{\prime\prime} g_i(x) + \exp g_i(x)\nabla_x g_i(x)\nabla_x g_i(x)^T\Bigg]}_A \\
& - \underbrace{\frac{1}{\sum\limits_{i=1}^m \exp g_i(x)}\Big(\sum\limits_{i=1}^m \exp g_i(x)\nabla_x g_i(x)\Big)\Big(\sum\limits_{i=1}^m \exp g_i(x)\nabla_x g_i(x)^T\Big)}_B
\end{array}
\end{equation}
To upper bound $\left\|\nabla_x^{\prime\prime} \bar{g}(x)\right\|_2$, which is 
\begin{equation}
\max\limits_{u^Tu = 1} u^T\nabla_x^{\prime\prime} \bar{g}(x)u = \max\limits_{u^Tu=1}u^TAu - u^TBu\le \max\limits_{u^Tu=1}u^TAu
\end{equation}
where the inequality is due to the fact that $B\succeq 0$.

Thus, we have $\left\|\nabla_x^{\prime\prime} \bar{g}(x)\right\|_2\le \left\|A\right\|_2$.
For the $\left\|A\right\|_2$, we have 
\begin{equation}
\begin{array}{ll}
\left\|A\right\|_2 =& \max\limits_{u^Tu = 1} u^TAu \le \frac{1}{\sum\limits_{i=1}^m \exp g_i(x)}\Big(\sum\limits_{i=1}^m\max\limits_{u^Tu = 1}
\exp g_i(x)u^T\nabla_x^{\prime\prime}g_i(x)u\Big) \\
&+\frac{1}{\sum\limits_{i=1}^m \exp g_i(x)}\Big(\sum\limits_{i=1}^m\max\limits_{u^Tu = 1}
\exp g_i(x)\left\|\nabla_x g_i(x)^Tu\right\|^2\Big) \\
&\le \frac{1}{\sum\limits_{i=1}^m \exp g_i(x)}\Big(\sum\limits_{i=1}^m \exp g_i(x)(L_i + \left\|\nabla_x g_i(x)\right\|^2)\Big)\\
&\le \frac{1}{\sum\limits_{i=1}^m \exp g_i(x)}\Big(\sum\limits_{i=1}^m \exp g_i(x)\Big)(\bar{L} + G^2) = \bar{L} + G^2
\end{array}
\end{equation}
where the first inequality comes from the optimality definition, the second inequality comes from the upper bound for each 
$\left\|\nabla_x^{\prime\prime}g_i(x)\right\|_2$ and the Cauchy - Schwartz inequality,
and the last inequality comes from the fact that $\bar{L} = \max L_i$ and $\left\|\nabla_x g_i(x)\right\|$ is upper bounded by $G$.
Thus, the last condition is also satisfied.
\end{proof}

\begin{proof}[Proof of the Proposition \ref{prop::similarResultTo2012}]
From Theorem \ref{thm::sumOfSquareLongterm}, we know that $\sum\limits_{t=1}^T\Big([g_i(x_t)]_+\Big)^2 \le O(\sqrt{T})$.
By using the inequality $(y_1+y_2+...+y_n)^2\le n(y_1^2+y_2^2+...+y_n^2)$,
setting $y_i$ being equal to $[g_i(x_t)]_+$, and $n = T$, we  have 
$\Big(\sum\limits_{t=1}^T[g_i(x_t)]_+\Big)^2 \le T\sum\limits_{t=1}^T\Big([g_i(x_t)]_+\Big)^2 \le O(T^{3/2})$.
Then we obtain that $\sum\limits_{t=1}^T[g_i(x_t)]_+ \le O(T^{3/4})$. 
Because $g_i(x_t)\le [g_i(x_t)]_+$, we also have $g_i(x_t) \le O(T^{3/4})$.
\end{proof}

\begin{proof}[Proof of the Proposition \ref{prop::tradeOffLossAndConstraint}]

Since we only change the stepsize for Algorithm \ref{alg::alg1}, 
the previous result in Lemma \ref{lem:sumOfLagfunction} and part of the proof 
up to Eq.(\ref{eq::UpperBoundOfsumOfObjAndLongTermConstrain}) in Theorem \ref{thm::sumOfSquareLongterm}
can be used without any changes.

First, let us rewrite the Eq.(\ref{eq::UpperBoundOfsumOfObjAndLongTermConstrain}):
\begin{equation}
\label{eq::key_our}
\small
\begin{array}{rl}
\sum\limits_{t=1}^T\Big(f_t(x_t) - f_t(x^*)\Big) & +
\sum\limits_{i=1}^m\sum\limits_{t=1}^T\frac{([g_i(x_t)]_+)^2}{\sigma\eta}\Big(1-\frac{(m+1)G^2}{2\sigma}\Big) \\
&\le \frac{R^2}{2\eta}+\frac{\eta T}{2}(m+1)G^2

\end{array}
\end{equation}

By plugging in the definition of $\alpha$, $\eta$, and that $\frac{([g_i(x_t)]_+)^2}{\sigma\eta}\alpha\ge 0$,
we have
\begin{equation}
\begin{array}{ll}
\sum\limits_{t=1}^T\Big(f_t(x_t)-f_t(x^*)\Big)&\le \frac{R^2}{2}T^{\beta}+\frac{(m+1)G^2}{2}T^{1-\beta} \\
                                              &=O(T^{max\{\beta,1-\beta\}}) 
\end{array}
\end{equation}

As argued in the proof of Theorem \ref{thm::sumOfSquareLongterm},
we have the following inequality with the help of
$\sum\limits_{t=1}^T\Big(f_t(x_t) - f_t(x^*)\Big)\ge -FT$:
\begin{equation}
\label{eq::constrain_ineq_our}
\begin{array}{l}
\sum\limits_{i=1}^m\sum\limits_{t=1}^T\frac{([g_i(x_t)]_+)^2}{\sigma\eta}\alpha 
\le \frac{R^2}{2}T^{\beta} + \frac{(m+1)G^2}{2}T^{1-\beta}+FT \\
\sum\limits_{t=1}^T([g_i(x_t)]_+)^2\le \frac{\sigma}{\alpha}(\frac{R^2}{2}+\frac{(m+1)G^2}{2}T^{1-2\beta}+FT^{1-\beta})

\end{array}
\end{equation}

Then we have 
\begin{equation}
\begin{array}{l}
\sum\limits_{t=1}^T[g_i(x_t)]_+\le\sqrt{T\sum\limits_{t=1}^T\Big([g_i(x_t)]_+\Big)^2} \\
\le\sqrt{\frac{T\sigma}{\alpha}\Big(\frac{R^2}{2}+\frac{(m+1)G^2}{2}T^{1-2\beta}+FT^{1-\beta}\Big)}\\
=O(T^{1-\beta/2})
\end{array}
\end{equation}
\end{proof}

It is also interesting to figure out why \cite{mahdavi2012trading} 
cannot have this user-defined trade-off benefit.
From \cite{mahdavi2012trading}, the key inequality in obtaining their conclusions is:
\begin{equation}
\label{eq::key_2012}
\begin{array}{l}
\sum\limits_{t=1}^T \Big(f_t(x_t)-f_t(x^*)\Big)+
\sum\limits_{i=1}^m\frac{\Big[\sum\limits_{t=1}^Tg_i(x_t)\Big]_+^2}{2(\sigma\eta T+m/\eta)}\\
\le \frac{R^2}{2\eta} + \frac{\eta T}{2}\Big((m+1)G^2+2mD^2\Big)
\end{array}
\end{equation} 
The main difference between Eq.(\ref{eq::key_2012}) and Eq.(\ref{eq::key_our})
is in the denominator of $\frac{\Big[\sum\limits_{t=1}^Tg_i(x_t)\Big]_+^2}{2(\sigma\eta T+m/\eta)}$.
Eq.(\ref{eq::key_2012}) has the form $(\sigma\eta T+m/\eta)$, 
while Eq.(\ref{eq::key_our}) has the form $(\sigma\eta)$. 
The coupled $\eta$ and $1/\eta$ prevents Eq.(\ref{eq::key_2012}) from arriving this user-defined trade-off.

The next proofs of the Proposition \ref{prop::true_violation_bound_2011} and \ref{prop::true_violation_bound_2016}
show how we can use our proposed Lagrangian function in Eq.(\ref{eq::new long term lagrangian})
to make the algorithms in \cite{mahdavi2012trading} and \cite{jenatton2016adaptive}
to have the clipped long-term constraint violation bounds.

\begin{sproof}[Proof of the Proposition \ref{prop::true_violation_bound_2011}]

If we look into the proof of Lemma 2 and Proposition 3 in \cite{mahdavi2012trading}, 
the new Lagrangian formula does not lead to any difference,
which means that the $\mathcal{L}_t(x,\lambda)$ defined in Eq.(\ref{eq::new long term lagrangian})
is also valid for the drawn conclusions.
Then in the proof of Theorem 4 in \cite{mahdavi2012trading}, we can change $g_i(x_t)$
to $[g_i(x_t)]_+$. 
The maximization for $\lambda$ over the range $[0,+\infty)$ is also valid, 
since $[g_i(x_t)]_+$ automatically satisfies this requirement.
Thus, the claimed bounds hold.
\end{sproof}

\begin{sproof}[Proof of the Proposition \ref{prop::true_violation_bound_2016}]

The previous augmented Lagrangian formula $\mathcal{L}_t(x,\lambda)$ used in \cite{jenatton2016adaptive} is:
\begin{equation}
\label{eq::old_lag_2016}
\mathcal{L}_t(x,\lambda) = f_t(x)+\lambda g(x) - \frac{\theta_t}{2}\lambda^2
\end{equation}
The Lemma 1 in \cite{jenatton2016adaptive} is the upper bound of $\mathcal{L}_t(x_t,\lambda) - \mathcal{L}_t(x_t,\lambda_t)$.
The proof does not make any difference between formula (\ref{eq::old_lag_2016}) and (\ref{eq::new_lag_2016}).
So we can still have the same conclusion of Lemma 1.
The Lemma 2 in \cite{jenatton2016adaptive} is the lower bound of $\mathcal{L}_t(x_t,\lambda) - \mathcal{L}_t(x^*,\lambda_t)$.
Since it only uses the fact that $g(x^*)\le 0$, which is also true for $[g(x^*)]_+$,
we can have the same result with $g(x_t)$ being replaced with $[g(x_t)]_+$.
The Lemma 3 in \cite{jenatton2016adaptive} is free of $\mathcal{L}_t(x,\lambda)$ formula, so it is also true for the new formula.
The Lemma 4 in \cite{jenatton2016adaptive} is the result of Lemma 1-3, 
so it is also valid if we change $g(x_t)$ to $[g(x_t)]_+$.
Then the conclusion of Theorem 1 in \cite{jenatton2016adaptive} is valid for $[g(x_t)]_+$ as well.
\end{sproof}

\end{document}